\documentclass{article}

\usepackage{graphicx}
\usepackage[ruled,vlined]{algorithm2e}
\usepackage{adjustbox}
\usepackage{comment}
\usepackage{caption}
\usepackage{subcaption}
\usepackage{sidecap}
\usepackage{wrapfig,lipsum,booktabs}

\usepackage{selectp}

\usepackage[preprint]{neurips_2021}

\usepackage[utf8]{inputenc} 
\usepackage[T1]{fontenc}    
\usepackage{hyperref}       
\usepackage{url}            
\usepackage{booktabs}       
\usepackage{amsfonts}       
\usepackage{nicefrac}       
\usepackage{microtype}      
\usepackage{xcolor}         

\usepackage{Definitions}

\newcommand{\tf}{{\tilde f}}

\newcommand{\tC}{{\tilde C}}
\newcommand{\bCcal}{{\bar \Ccal}}
\newcommand{\tpsi}{{\tilde \psi}}

\newcommand{\tLcal}{{\tilde \Lcal}}
\newcommand{\highlight}[1]{\colorbox{blue!10}{#1}}

\DeclareMathOperator{\sg}{sg}

\title{Discrete-Valued Neural Communication}

\author{
    Dianbo Liu* \\
     Mila \\
     \And
    Alex Lamb* \\
    Mila \\
     \And
    Kenji Kawaguchi \\
    Harvard University \\
    \And
    Anirudh Goyal\\
    Mila\\
    \And
    Chen Sun \\
    Mila \\
     \And
     Michael Curtis Mozer \\
    Google Brain and University of Colorado \\
    \And
     Yoshua Bengio \\
    Mila \\
    \And
\texttt{* co-first author} \\
\texttt{Emails:liudianbo@gmail.com, alex6200@gmail.com, kkawaguchi@fas.harvard.edu} \\

}

\begin{document}
\maketitle

\begin{abstract}
Deep learning has advanced from fully connected architectures to structured models organized into components, e.g., the transformer composed of positional elements, modular architectures divided into slots, and graph neural nets made up of nodes. In structured models, an interesting question is how to conduct dynamic and possibly sparse communication among the separate components. Here, we explore the hypothesis that restricting the transmitted information among components to discrete representations is a beneficial bottleneck. The motivating intuition is human language in which communication occurs through discrete symbols. Even though individuals have different understandings of what a ``cat'' is based on their specific experiences, the shared discrete token makes it possible for communication among individuals to be unimpeded by individual differences in internal representation.  To discretize the values of concepts dynamically communicated among specialist components, we extend the quantization mechanism from the Vector-Quantized Variational Autoencoder to multi-headed discretization with shared codebooks and use it for \emph{discrete-valued neural communication (DVNC)}. Our experiments show that \emph{DVNC} substantially improves systematic generalization in a variety of architectures---transformers, modular architectures, and graph neural networks. We also show that the DVNC is robust to the choice of hyperparameters, making the method very useful in practice. Moreover, we establish a theoretical justification of our discretization process, proving that it has the ability to increase noise robustness and reduce the underlying dimensionality of the model. 
\end{abstract}

\section{Introduction} \label{sec:intro.}


In AI, there has long been a tension between subsymbolic and symbolic architectures. Subsymbolic architectures, like
neural networks, utilize continuous representations and statistical computation. Symbolic architectures, like production
systems \citep{laird1986chunking} and traditional expert systems, use discrete, structured representations and logical computation. Each architecture has its strengths: subsymbolic computation is useful for perception and control, symbolic computation for higher level,
abstract reasoning. A challenge in integrating these approaches is developing unified learning procedures.

As a step toward bridging the gap, recent work in deep learning has focused on constructing structured architectures with multiple components that interact with one another. For instance, graph neural networks  are composed of distinct nodes \citep{kipf2019contrastive,scarselli2008graph,kipf2018neural,santoro2017simple,raposo2017discovering,bronstein2017geometric,gilmer2017neural,tacchetti2018relational,van2018relational}, transformers are composed of positional elements \citep{bahdanau2014neural,vaswani2017attention}, and modular 
models are divided into slots or modules with bandwidth limited communication \citep{jacobs1991adaptive,bottou1991framework,goyal2020inductive,ronco1996modular,reed2015neural, lamb2021transformers,andreas2016neural,rosenbaum2017routing,fernando2017pathnet,shazeer2017outrageously,rosenbaum2019routing}.

Although these structured models exploit the discreteness in their architectural components, the present work extends these models to leverage discreteness of representations, which is an essential property of symbols. We propose to learn a common \emph{codebook} that is shared by all components for inter-component communication. The codebook permits only a discrete set of  communicable values. We hypothesize that this  communication based on the use and reuse of discrete symbols will provide us with two benefits:
\begin{itemize}
\item The use of discrete symbols limits the bandwidth of representations whose meaning needs to be learned and synchronized across modules. It may therefore
serve as a common language for interaction, and make it easier to learn.
\item The use of shared discrete symbols will
promote systematic generalization by allowing for the reuse of previously encountered symbols in new situations. This makes it easier to hot-swap one component 
for another when new out-of-distribution (OOD) settings arise that require combining existing components in novel ways.
\end{itemize}
Our work is inspired by cognitive science, neuroscience, and mathematical considerations. From the cognitive science
perspective, we can consider different components of structured neural architectures to be analogous to
autonomous agents in a distributed system whose ability to communicate stems from sharing the same language.
If each agent speaks a different language, learning to communicate would be slow and past experience would
be of little use when the need to communicate with a new agent arises. If all agents learn the same language,
each benefits from this arrangement. To encourage a common language, we limit the expressivity of the vocabulary
to discrete symbols that can be combined combinatorially. From the neuroscience perspective, we note that various areas in the brain, including the hippocampus \citep{sun2020hipp, quiroga2005, wills2005}, the prefrontal cortex  \citep{Fujii1246}, and sensory cortical areas  \citep{Tsao670} are tuned to discrete variables (concepts, actions, and objects), suggesting the evolutionary advantage of such encoding, and its contribution to the capacity for generalization in the brain. From a theoretical perspective, we present analyses
suggesting that multi-head discretization of inter-component communication increases model sensitivity and reduces underlying dimensions 
(Section \ref{sec:theory}). These sources of inspiration lead us to the proposed method of 
\emph{Discrete-Valued Neural Communication (DVNC)}.

Architectures like graph neural networks (GNNs), transformers, and slot-based or modular neural networks consist of articulated specialist components, for instance,  nodes in GNNs, positions in transformers, and slots/modules for modular models. We evaluate the efficacy of DVNC in GNNs, transformers, and in a modular recurrent architecture called RIMs. For each of these structured architectures, we keep the original architecture and all of its specialist components the same. The only change is that we impose discretization in the communication between components (Figure~\ref{fig:Implementation}). 

Our work is organized as follows. First, we introduce DVNC and present theoretical analyses showing that DVNC improves sensitivity and reduces underlying dimensionality of models (the logarithm of the covering number). Then we explain how DVNC can be incorporated into different model architectures. And finally we report experimental results showing improved OOD generalization with DVNC.  

\section{Discrete-Value Neural Communication and Theoretical Analysis} 
\label{sec:theory}

In this section, we begin with the introduction of Discrete-Value Neural Communication (DVNC) and proceed by conducting a theoretical analysis of DVNC affects the sensitivity and underlying dimensionality of models. We then explain how DVNC can be used within several different architectures. 


\paragraph{Discrete-Value Neural Communication (DVNC)} 

The process of converting data with continuous attributes into data with discrete attributes is called discretization \citep{chmielewski1996global}. In this study, we use discrete latent variables to quantize information communicated among different modules in a similar manner as in Vector Quantized Variational AutoEncoder (VQ-VAE) \citep{oord2017neural}. Similar to VQ-VAE, we introduce a discrete latent space vector $e \in \RR^{L \times m}$ where $L$ is the size of the discrete latent space (i.e., an $L$-way categorical variable), and $m$ is the dimension of each latent embedding vector $e_j$. Here, $L$ and $m$  are both hyperparameters. In addition, by dividing each target vector into $G$ segments or discretization heads, we separately quantize each head and concatenate the results (Figure~\ref{fig:Discretization}). More concretely, the discretization process for each vector $h  \in \Hcal\subset \RR^{m}$ is described as follows. First, we divide a vector $h$ into $G$ segments $s_1,s_2,\dots,s_G$ with $h=\textsc{concatenate}(s_1,s_2,\dots,s_G),$ where each segment $s_i \in \RR^{m/G}$ with $\frac{m}{G} \in \NN^+$. Second, we discretize each segment $s_i$ separately: 
$$
e_{o_i}=\textsc{discretize}(s_i), \quad \text{ where } o_i=\argmin_{j \in \{1,\dots,L\}} ||s_{i}-e_j||.
$$
Finally, we concatenate the discretized results to obtain the final discretized vector $Z$ as
\[Z=\textsc{concatenate}(\textsc{discretize}(s_1),\textsc{discretize}(s_2),...,\textsc{discretize}(s_G)).\]
The multiple steps described above can be summarized by $Z=q(h,L,G)$, 
where $q(\cdot)$ is the whole discretization process with the codebook, $L$ is the codebook size, and $G$\ is number of segments per vector. It is worth emphasizing that the codebook $e$ is shared across all communication vectors and heads, and is trained together with other parts of the model.

\begin{figure}
     \centering
     \begin{subfigure}[b]{0.26\textwidth}
         \centering
         \includegraphics[width=\textwidth]{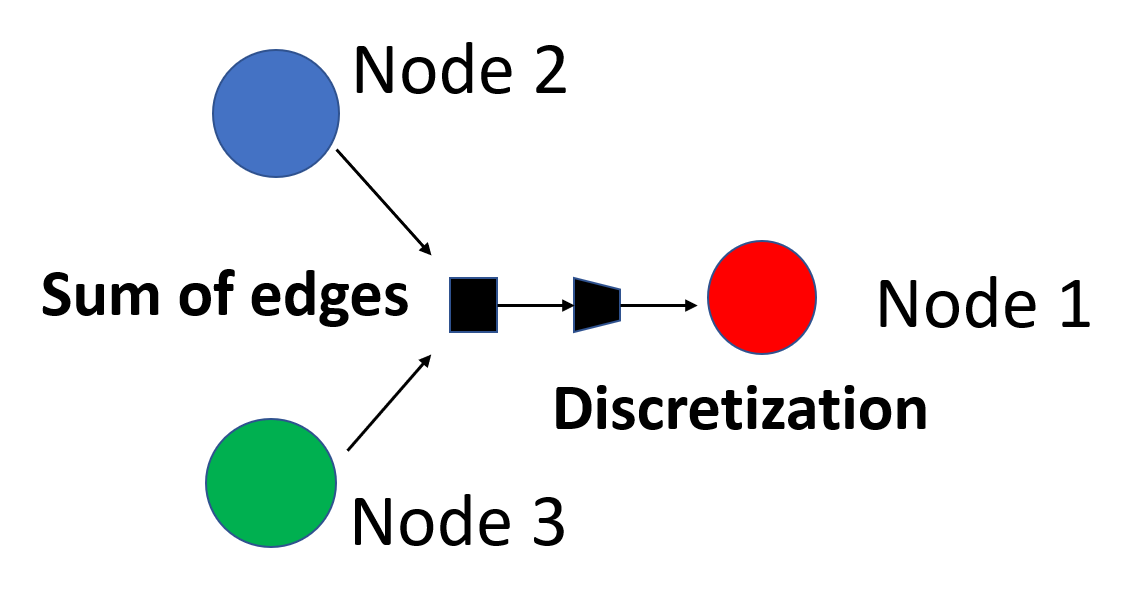}
         \caption{DVNC in Graph Neural Network}
     \end{subfigure}
     \hfill
     \begin{subfigure}[b]{0.35\textwidth}
         \centering
         \includegraphics[width=\textwidth]{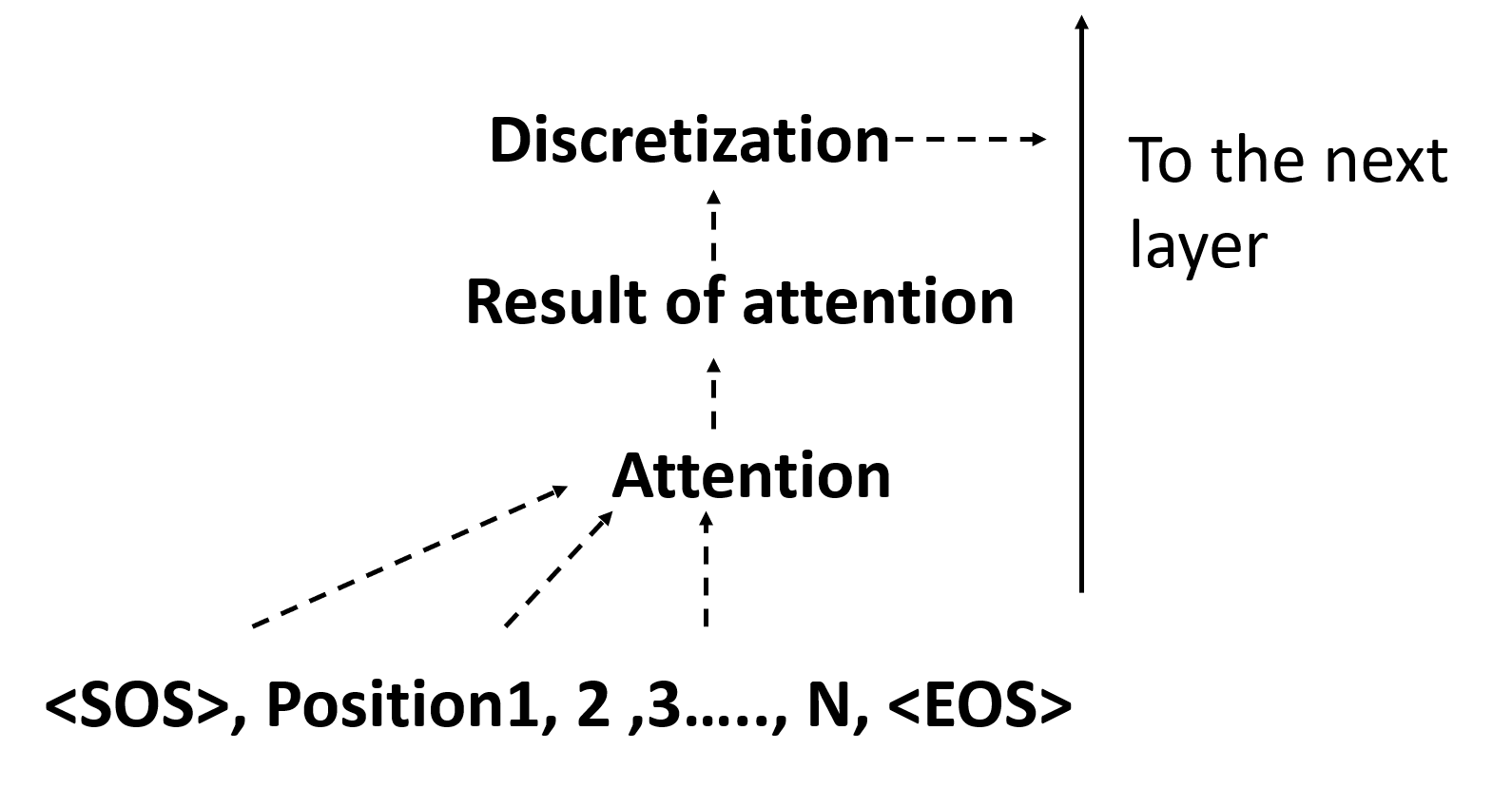}
         \caption{DVNC in Transformer}
     \end{subfigure}
     \hfill
     \begin{subfigure}[b]{0.37\textwidth}
         \centering
         \includegraphics[width=\textwidth,trim={0 0 0 2cm},clip]{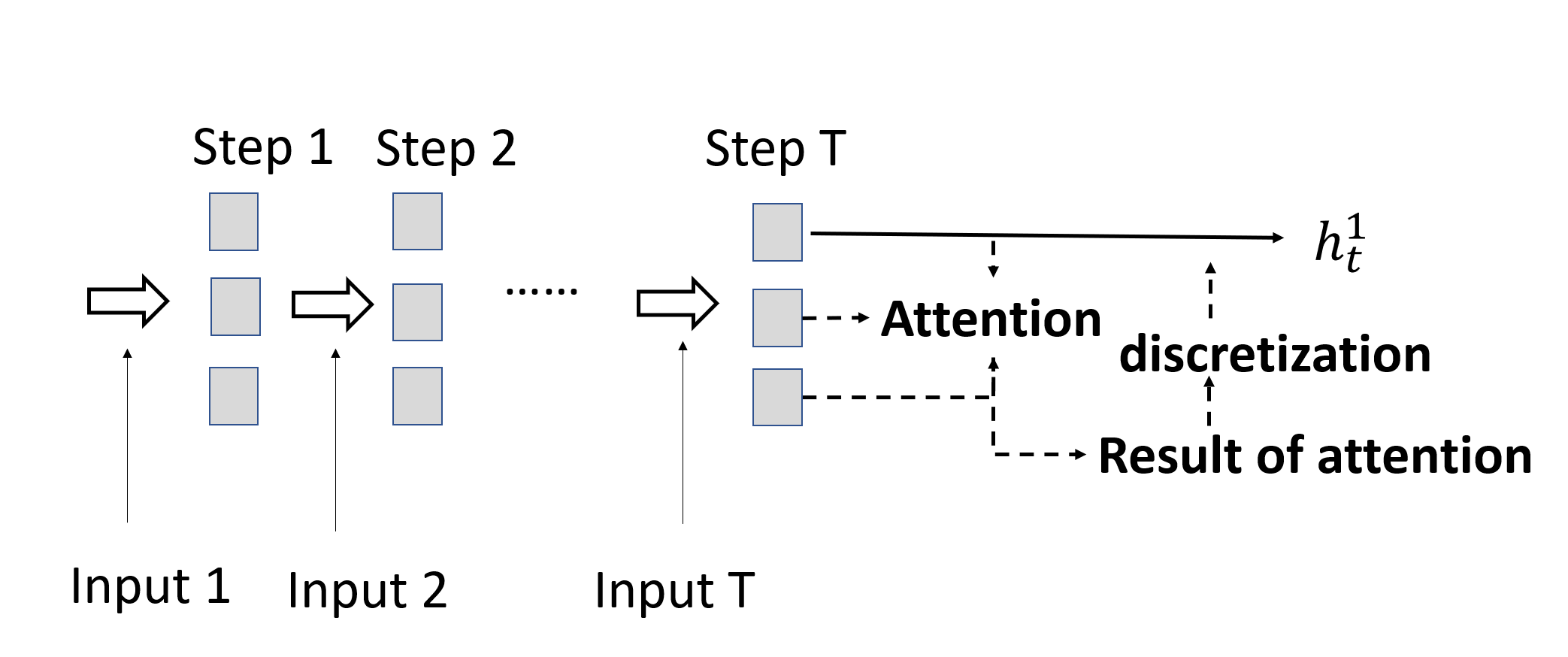}
         \caption{DVNC in Modular Recurrent Mechanisms (RIMs)}
     \end{subfigure}
        \caption{Communication among different components in neural net models is discretized via a shared codebook. In modular recurrent neural networks and transformers, values of results of attention are discretized. In graph neural networks, communication from edges is discretized. }
\label{fig:Implementation} 
\end{figure}

The overall loss for model training is:
\begin{equation}\label{equation2}
\mathcal{L}=\mathcal{L}_\mathrm{task}+\frac{1}{G}\left(\sum^{G}_{i}||\sg(s_{i})-e_{o_i}||^2_2+\beta \sum^{G}_{i}||s_{i}-\sg(e_{o_i})||^2_2\right)
\end{equation}

where $\mathcal{L}_\mathrm{task}$ is the loss for specific task, \textit{e.g.}, cross entropy loss for classification or mean square error loss for regression, $\sg$ refers to a stop-gradient operation that blocks gradients from flowing into its argument, and $\beta$ is a hyperparameter which controls the reluctance to change the code. The second term $\sum^{G}_{i}||\sg(s_{i})-e_{o_i}||^2_2$ is the codebook loss, which only applies to the discrete latent vector and brings the selected selected $e_{o_i}$ close to the output segment $s_i$. The third term $\sum^{G}_{i}||s_{i}-\sg(e_{o_i})||^2_2$ is the commitment loss, which only applies to the target segment $s_i$ and trains the module that outputs $s_i$ to make $s_i$ stay close to the chosen discrete latent vector $e_{o_i}$. We picked $\beta=0.25$ as in the original VQ-VAE paper \citep{oord2017neural}. We initialized $e$ using $k$-means clustering on vectors $h$ with $k=L$ and trained the codebook together with other parts of the model by gradient decent. When there were multiple $h$ vectors to discretize in a model, the mean of the codebook and commitment loss across all $h$ vectors was used. Unpacking this equation, it can be seen that we adapted the vanilla VQ-VAE loss to directly suit our discrete communication method \citep{oord2017neural}. In particular, the VQ-VAE loss was adapted to handle multi-headed discretization by summing over all the separate discretization heads.

\sidecaptionvpos{figure}{c}

\begin{SCfigure}
  \centering 
  \includegraphics[width=0.50\linewidth,trim={8cm 2cm 3cm 0},clip]{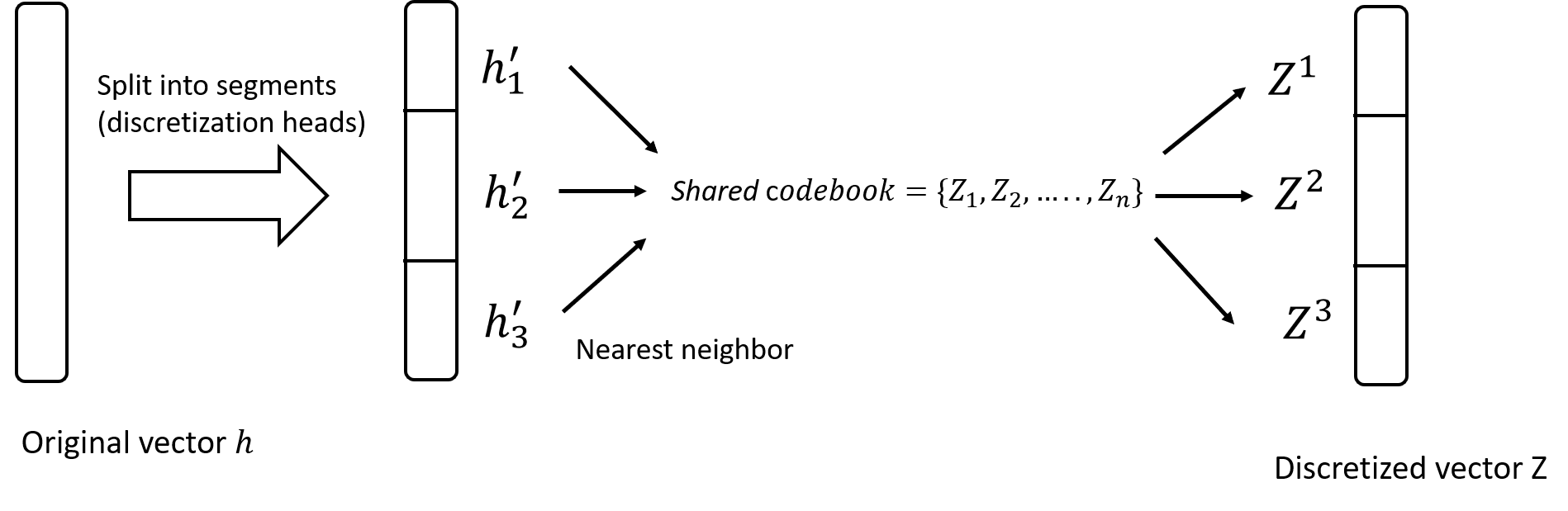} 
  \caption{In structured architectures, communication is typically vectorized. In DVNC, this communication vector is first divided into discretization heads.  Each head is discretized separately to the nearest neighbor of a collection of latent codebook vectors which is shared across all the heads.  The discretization heads are then concatenated back into the same shape as the original vector.}
  \label{fig:Discretization} 
\end{SCfigure}

\begin{table}[th!]
\caption{Communication with discretized values achieves a noise-sensitivity bound that is independent of the number of dimensions $m$ and network lipschitz constant $\bar \varsigma_{k}$ and only depends on the number of discretization heads $G$ and codebook size $L$.  }
\label{tab:motivation}
\resizebox{\linewidth}{!}{%
\begin{tabular}{|c|c|c|}
\hline
Communication Type & Example & Sensitivity Bounds (Thm~\ref{prop:1},~\ref{prop:2}) \\\hline
\begin{tabular}[c]{@{}c@{}}Communication with continuous signals\\ is expressive but can take a huge range of novel\\ values, leading to poor systematic generalization \end{tabular} & \begin{tabular}[c]{@{}c@{}}\includegraphics[width=0.12\linewidth]{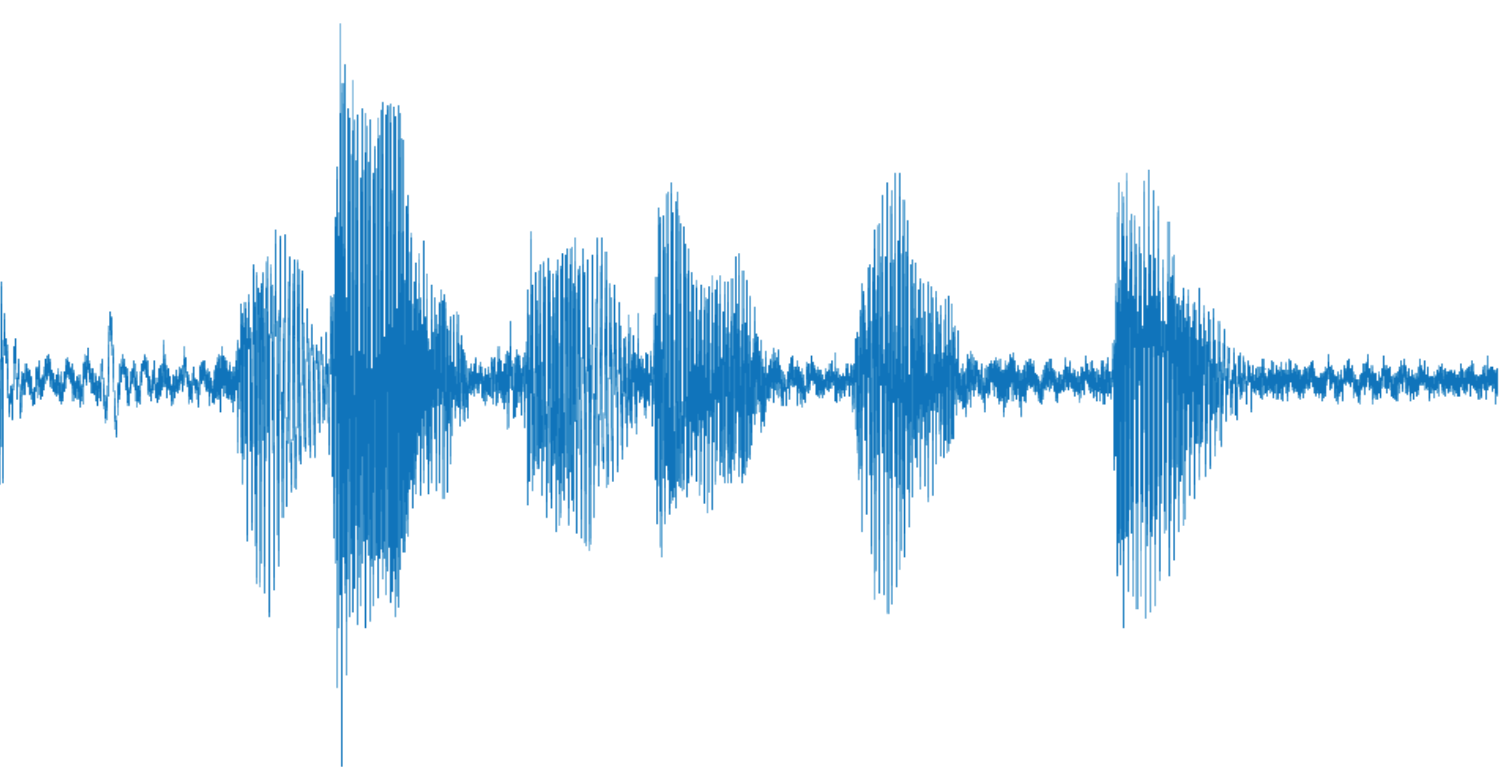} \\ $m \sim 10^5$ \end{tabular} & $\Ocal\left( \sqrt{\frac{m\ln(4  \sqrt{n m})+\ln(2/\delta)}{2n}}+\frac{\bar \varsigma_{k}R_{\Hcal}}{\sqrt{n}}\right)$ \\ \hline
\begin{tabular}[c]{@{}c@{}}Communication with multi-head \\ discrete-values is both expressive and sample efficient\end{tabular} & \begin{tabular}[c]{@{}c@{}}``John owns a car'' \\ $G = 15$ \\ $L = 30$ \end{tabular}   & $\Ocal\left(\sqrt{\frac{G\ln(L)+\ln(2/\delta)}{2n}}\right)$ \\ \hline
\end{tabular}%
}
\end{table}


 In the next subsection, we use the following additional notation. The function $\phi$ is arbitrary and thus can refer to the composition of an evaluation criterion and the rest of the network following discretization.  Given any function $\phi: \RR^{m} \rightarrow \RR$ and any family of sets $S=\{S_1,\dots,S_K\}$ with $S_1,\dots,S_K \subseteq \Hcal$, we define the corresponding function $\phi_{k}^S$ by $\phi_{k}^S(h)=\one\{h \in S_k\}\phi(h)$ for all $k \in [K]$, where $[K]=\{1,2,\dots,K\}$.
Let $e \in E\subset \RR^{L\times m}$ be fixed and we denote by $(Q_k)_{k \in [L^G]}$ all the possible values after the discretization process: i.e.,  $q(h,L,G)\in \cup_{k \in [L^G]} \{Q_k\}$ for all $h \in \Hcal$.

\paragraph{Theoretical Analysis} 
  This subsection shows that adding the discretization process has two potential advantages: (1) it improves noise robustness and (2) it reduces underlying dimensionality.  These are proved in Theorems \ref{prop:1}--\ref{prop:2}, illustrated by examples (Table~\ref{tab:motivation}), and explored in analytical experiments using Gaussian-distributed vectors (Figure~\ref{fig:empirical_analysis}).  
  
  To understand the advantage on noise robustness, we note that there is  an additional error incurred by  noise \textit{without discretization}, i.e., the second term $\bar \varsigma_{k}R_{\Hcal}/\sqrt{n}\ge0$ in the bound of Theorem \ref{prop:2} ($\bar \varsigma_{k}$ and $R_{\Hcal}$ are defined in  Theorem \ref{prop:2}). This  error due to  noise disappears \textit{with discretization} in the bound of Theorem \ref{prop:1} as the discretization process reduces the sensitivity to noise. This is because the discretization process lets the communication become invariant to noise  within the same category; e.g.,  the communication is invariant to different notions  of ``cats''. 
  
  To understand the advantage of discretization on dimensionality, we can see that it reduces the underlying dimensionality of $m\ln(4\sqrt{n m})$ \textit{without discretization} (in  Theorem \ref{prop:2}) to that of $G\ln(L)$ \textit{with discretization} (in Theorem \ref{prop:1}).  As a result, the size of the codebook $L$ affects the underlying dimension in a weak (logarithmic) fashion, while the number of dimensions $m$ and the number of discretization heads $G$ scale the underlying dimension in a linear way. Thus, the discretization process successfully lowers the underlying dimensionality for any $n \geq 1$ as long as $G\ln(L) < m\ln(4\sqrt{m})$.  This is nearly always the case as the number of discretization heads $G$ is almost always much smaller than the number of units $m$.  Intuitively, a discrete language has combinatorial expressiveness, making it able to model complex phenomena, but still lives in a much smaller space than the world of unbounded continuous-valued signals (as $G$ can be much smaller than $m$).

\begin{theorem} \label{prop:1}
\emph{(with discretization)} Let $S_{k}=\{Q_{k}\}$ for all $k \in [L^{G}]$. Then, for any $\delta>0$, with probability at least $1-\delta$ over an iid draw of $n$ examples $(h_{i})_{i=1}^n$, the following holds for any $\phi: \RR^{m} \rightarrow \RR$ and all $k \in [L^{G}]$: if $|\phi_k^S(h)|\le \alpha$ for all $h \in \Hcal$, then 
\begin{align} \label{eq:prop:1:1}
\left| \EE_{h}[\phi_{k}^S(q(h,L,G))] - \frac{1}{n}\sum_{i=1}^n\phi_{k}^S(q(h_{i}, L,G)) \right| = \Ocal\left( \alpha \sqrt{\frac{G\ln(L)+\ln(2/\delta)}{2n}}\right), 
\end{align}
where no constant is hidden in $\Ocal$.
\end{theorem} 
\begin{theorem} \label{prop:2} 
\emph{(without discretization)}  Assume that  $\|h\|_2 \le R_{\Hcal}$ for all $h \in \Hcal\subset \RR^{m}$. Fix  $\Ccal\in \argmin_{\bCcal}\{|\bCcal|:\bCcal \subseteq \RR^{m},\Hcal \subseteq \cup_{c \in\bCcal} \Bcal_{}[c]\}$ where $\Bcal[c]=\{x\in \RR^{m} : \|x - c \|_2\le R_{\Hcal}/(2\sqrt{n})\}$. Let $S_{k}=\Bcal[c_{k}]$ for all $k \in [|\Ccal|]$ where $c_k \in \Ccal$ and $\cup_{k} \{c_k\}=\Ccal$.   Then, for any $\delta>0$, with probability at least $1-\delta$ over an iid draw of $n$ examples $(h_{i})_{i=1}^n$, the following holds  for any $\phi: \RR^{m} \rightarrow \RR$  and  all $k \in [|\Ccal|]$:
if $|\phi_k^S(h)|\le \alpha$ for all $h \in \Hcal$ and 
$|\phi_{k}^{S}(h{})-\phi_{k}^{S}(h{}')|\le \varsigma_{k}\|h -h'\|_2$ for all $h,h' \in S_{k}$, then
\begin{align} \label{eq:prop:2:1}
\left| \EE_{h}[\phi_{k}^S(h)] - \frac{1}{n}\sum_{i=1}^n\phi_{k}^S(h_{i}) \right| 
 = \Ocal\left(\alpha \sqrt{\frac{m\ln(4  \sqrt{n m})+\ln(2/\delta)}{2n}}+\frac{\bar \varsigma_{k}R_{\Hcal}}{\sqrt{n}}\right), 
\end{align}
where no constant is hidden in $\Ocal$ and $\bar \varsigma_{k}= \varsigma_{k} \left(\frac{1}{n} \sum_{i=1}^n\one\{h_{i}\in \Bcal_{}[c_{k}]\}\right)$.
\end{theorem} 

The two proofs of these theorems use the same steps and are equally tight as shown in Appendix \ref{app:proof}. Equation \eqref{eq:prop:2:1} is also as tight as that of related work as discussed in Appendix \ref{app:3}. The set $S$ is chosen to cover the original continuous space $\Hcal$ in Theorem \ref{prop:2} (via the $\epsilon$-covering $\Ccal$ of $\Hcal$), and its discretized space in Theorem \ref{prop:1}. Equations \eqref{eq:prop:1:1}--\eqref{eq:prop:2:1} hold for all functions $\phi: \RR^{m} \rightarrow \RR$, including the maps that depend on the samples $(h_{i})_{i=1}^n$ via any learning processes. For example, we can set $\phi$ to be an evaluation criterion  of the latent space $h$ or the composition of an evaluation criterion and any neural network layers that are learned with the samples $(h_{i})_{i=1}^n$. In  Appendix \ref{sec:2}, we present additional theorems, Theorems \ref{thm:1}--\ref{thm:2}, where we analyze the effects of learning the map $x\mapsto h$ and the codebook $e$ via input-target pair samples $((x_i,y_i))_{i=1}^n$. 

Intuitively, the proof shows that we achieve the improvement in sample efficiency when $G$ and $L$ are small, with the dependency on $G$ being significantly stronger (details in Appendix).  Moreover, the dependency of the bound on the Lipschitz constant $\varsigma_{k}$ is eliminated by using discretization.  Our theorems \ref{prop:1}--\ref{prop:2} are applicable to all of our models for recurrent neural networks, transformers, and graph neural networks (since the function $\phi$ is arbitrary) in the following subsections.


\begin{figure}
    \centering
    \includegraphics[width=0.24\linewidth]{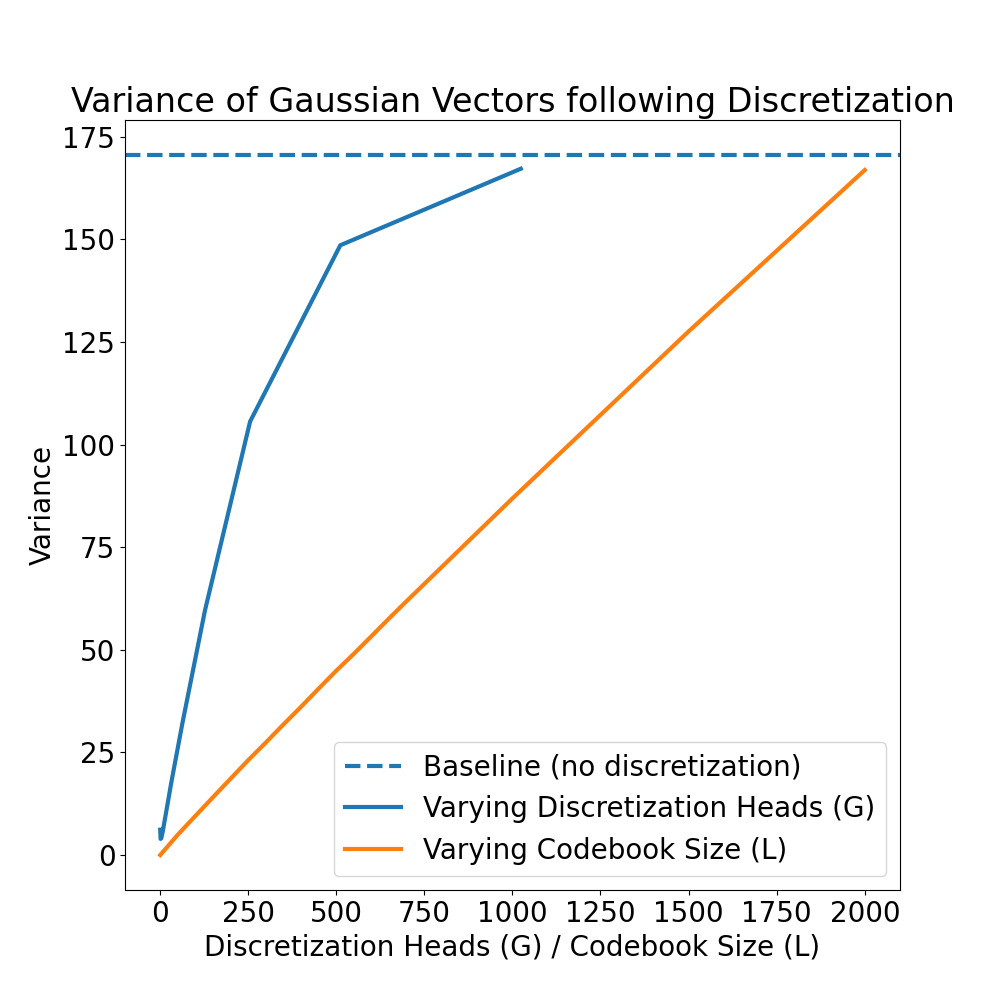}
    \includegraphics[width=0.24\linewidth]{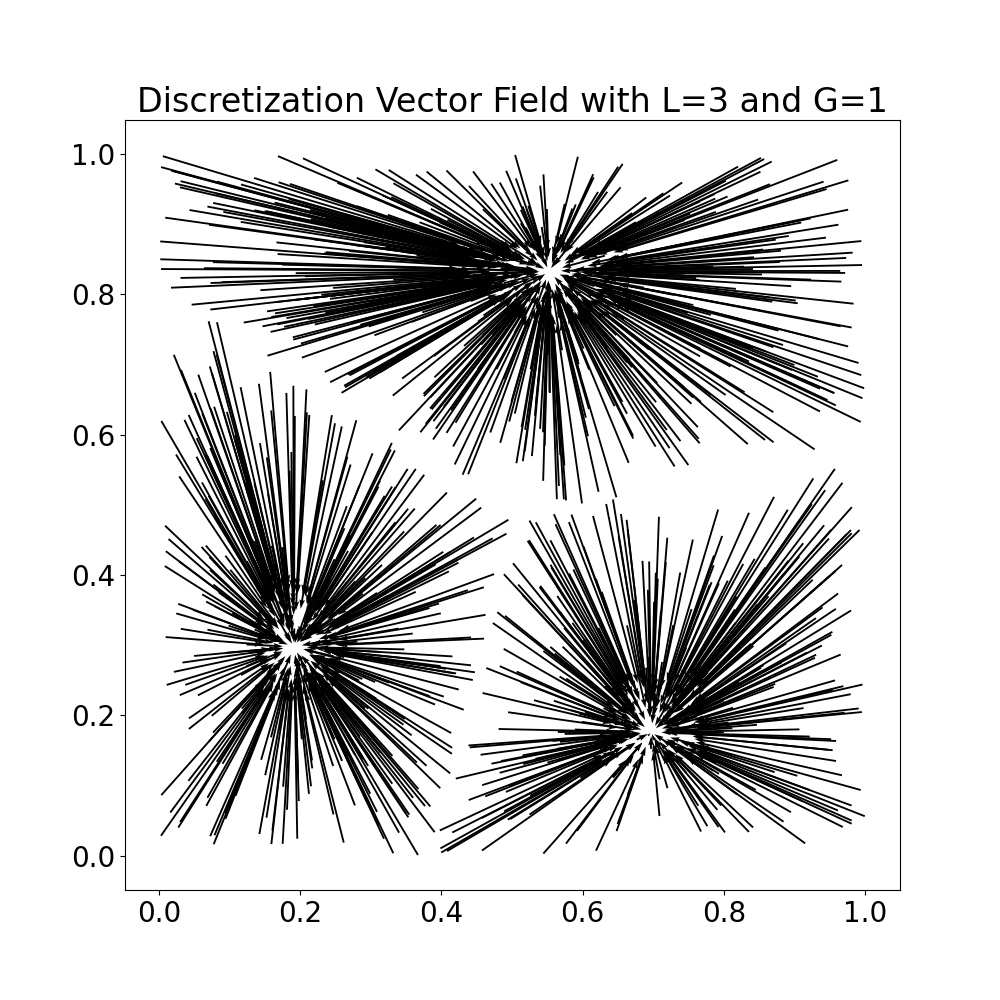}
    \includegraphics[width=0.24\linewidth]{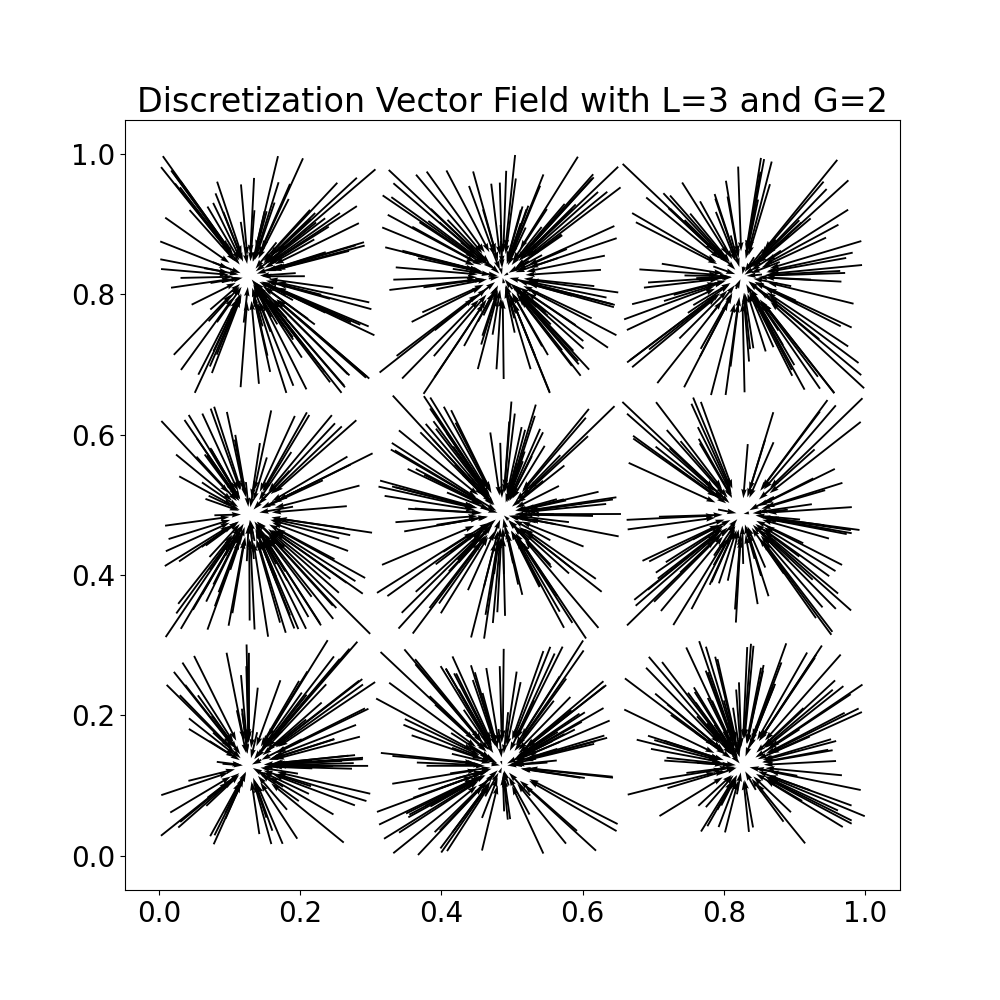}
    \includegraphics[width=0.24\linewidth]{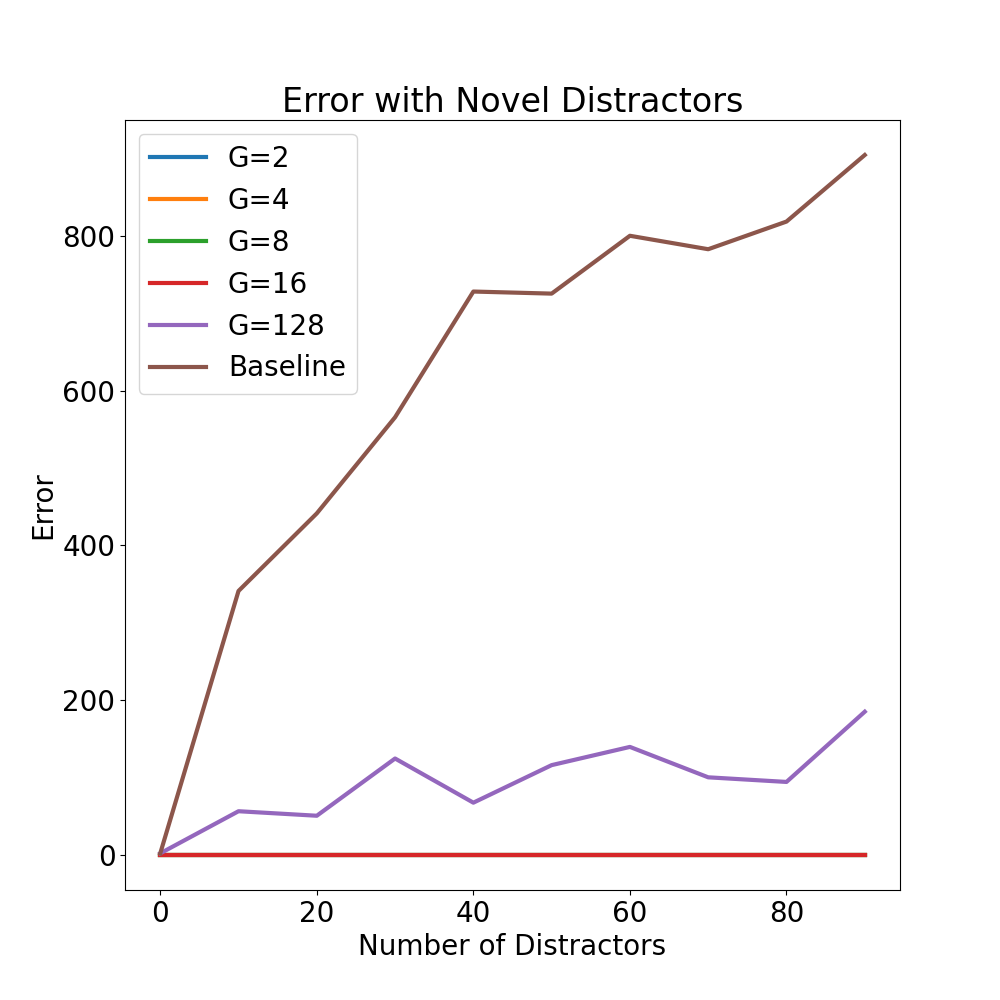}
    \caption{We perform empirical analysis on Gaussian vectors to build intuition for our theoretical analysis.  Expressiveness scales much faster as we increase discretization heads than as we increase the size of the codebook.  This can be seen when we measure the variance of a collection of Gaussian vectors following discretization (left), and can also be seen when we plot a vector field of the effect of discretization (center).  Discretizing the values from an attention layer trained to select a fixed Gaussian vector makes it more robust to novel Gaussian distractors (right). For more details, see Appendix~\ref{app:experiments}. }
    \label{fig:empirical_analysis}
\end{figure}

\paragraph{Communication along edges of Graph Neural Network} \label{sec:GNN}

One model where relational information is naturally used is in a graph neural network for modelling object interactions and predicting the next time frame. We denote node representations by $\zeta^t_i$, edge representations by $\epsilon^t_{i,j}$, and actions applied to each node by $a^t_i$. Without DVNC, the changes of each nodes after each time step is computed by $\zeta^{t+1}_i=\zeta^t_i+\Delta \zeta^t_i$, where $\Delta \zeta^t_i=f_{node}(\zeta^t_i,a^i_t,\sum_{j \neq i} \epsilon^{i,j}_t)$ and $\epsilon^t_{i,j}=f_{edge}(\zeta^t_i,z^t_j)$. 

In this present work, we discretize the sum of all edges connected to each node with DVNC, as so: $\Delta \zeta^t_i=f_{node}(\zeta^t_i,a^i_t,q(\sum_{j \neq i} \epsilon^{i,j}_t,L,G))$.

\paragraph{Communication Between Positions in Transformers} \label{sec:Transformer}

In a transformer model without DVNC, at each layer, the scaled dot product multi-head soft attention is applied to allow the model to jointly attend to information from different representation
subspaces at different positions \citep{vaswani2017attention} as: 
$$
\mathrm{Output}=\mathrm{residual}+\textsc{MultiHeadAttention}(B,K,V),
$$ 
where 
$\textsc{MultiHeadAttention}(B,K,V)=\textsc{Concatenate}(head_1,head_2,.....head_n)W^O$
and 
$head_i=\textsc{SoftAttention}(BW^{B}_{i},KW^{K}_{i},VW^{V}_{i})$.
Here, $W^O,W^B,W^K$, and $W^V$ are projection matrices, and $B$, $K$, and $V$ are queries, keys, and values respectively. 

In this present work, we applied the DVNC  process to the results of the attention in the last two layers of transformer model, as so:
\[\mathrm{Output}=\mathrm{residual}+q(\textsc{MultiHeadAttention}(B,K,V),L,G).\]

\paragraph{Communication with Modular Recurrent Neural Networks} \label{sec:RIM}
There have been many efforts to introduce modularity into RNN. Recurrent independent mechanisms (RIMs) activated different modules at different time step based on inputs \citep{goyal2019recurrent}. In RIMs, outputs from different modules are communicated to each other via soft attention mechanism. In the original RIMs method, we have $\hat{z}^{t+1}_{i}=\textsc{RNN}(z^{t}_{i},x^t)$ for active modules, and $\hat{z}^{t+1}_{i'}=z^{t}_{i'}$ for inactive modules, where $t$ is the time step, $i$ is index of the module, and $x_t$ is the input at time step $t$. Then, the dot product query-key soft attention is used to communication output from all modules $i \in \{1,\dots,M\}$ such that
$
h^{t+1}_i=\textsc{SoftAttention}(\hat{z}^{t+1}_{1},\hat{z}^{t+1}_{2},....\hat{z}^{t+1}_{M}).
$

In this present work, we applied the DVNC  process to the output of the soft attention, like so: 
$z^{t+1}_{i}=\hat{z}^{t+1}_{i}+q(h^{t+1}_i, L, G).$
Appendix~\ref{app:method} presents the pseudocode for RIMs with discretization.  







\section{Related Works}

\paragraph{The Society of Specialists}
Our work is related to the theoretical nature of intelligence proposed by \citet{minsky1988society} and others \citep{braitenberg1986vehicles,fodor1983modularity}, in which the authors suggest that an intelligent mind can be built from many little specialist parts, each mindless by itself. Dividing model architecture into different specialists been the subject of a number of research directions, including neural module networks \citep{andreas2016neural}), multi-agent reinforcement learning \citep{zhang2019multi} and many others \citep{jacobs1991adaptive,reed2015neural,rosenbaum2019routing}. Specialists for different computation processes have been introduced in many models such as RNNs and transformers \citep{goyal2019recurrent,lamb2021transformers,goyal2021coordination}. Specialists for entities or objects in fields including computer vision \citep{kipf2019contrastive}. Methods have also been proposed to taking both entity and computational process into consideration \citep{goyal2020object}. In a more recent work, in addition to entities and computations, rules were considered when designing specialists \citep{goyal2021neural}. Our Discrete-Valued Neural Communication method can be seen as introducing specialists of representation into machine learning model. 


\paragraph{Communication among Specialists}
Efficient communication among different specialist components in a model requires compatible representations and synchronized messages. In recent years, attention mechanisms are widely used for selectively communication of information among specialist components in machine learning modes \citep{goyal2019recurrent,goyal2021coordination,goyal2021neural} and transformers \citep{vaswani2017attention,lamb2021transformers}.collective memory and shared RNN parameters have also been used for multi-agent communication  \citep{garland1996multiagent,pesce2020improving}. Graph-based models haven been widely used in the context of relational reasoning, dynamical system simulation , multi-agent systems and many other fields. In graph neural networks, communication among different nodes were through edge attributes that a learned from the nodes the edge is connecting together with other information \citep{kipf2019contrastive,scarselli2008graph, bronstein2017geometric, watters2017visual,  van2018relational, kipf2018neural, battaglia2018relational, tacchetti2018relational,op3veerapaneni}.Graph based method represent entities, relations, rules and other elements as node, edge and their attributes in a graph \citep{koller2009probabilistic,battaglia2018relational}. In graph architectures, the inductive bias is assuming the system to be learnt can be represented as a graph. In this study, our DVNC introduces inductive bias that forces inter-component communication to be discrete and share the same codebook. The combinatorial properties come from different combinations of latent vectors in each head and different combination of heads in each representation vector.While most of inter-specialist communication mechanisms operates in a pairwise symmetric manner, 
\citet{goyal2021coordination} introduced a bandwidth limited communication channel to allow information from a limited number of modules to be broadcast globally to all modules, inspired by Global workspace theory \citep{baars2019consciousness,baars1993cognitive}.our proposed method selectively choose what information can communicated from each module. We argue these two methods are complimentary to each other and can be used together, which we like to investigate in future studies.





\begin{table}
\caption{Performance of Transformer Models with Discretized Communication on the Sort-of-Clevr Visual Reasoning Task.}
\label{Table:TransformerResults}
\resizebox{1.0\linewidth}{!}{
\begin{tabular}{c c c c}\hline
\textbf{Method} & \textbf{Ternary Accuracy} & \textbf{Binary Accuracy} & \textbf{Unary Accuracy} \\ 
\midrule
Transformer baseline                 & 57.25 $\pm$ 1.30                             & 76.00 $\pm$ 1.41                            & 97.75 $\pm$ 0.83                           \\ 
Discretized transformer (G=16)       & 61.33 $\pm$ 2.62                             & 84.00 $\pm$ 2.94                            & 98.00 $\pm$ 0.89                           \\ 
Discretized transformer (G=8)        & \highlight{62.67 $\pm$ 1.70}                             & \highlight{88.00 $\pm$ 0.82}                            & \highlight{98.75 $\pm$ 0.43}                           \\ 
Discretized transformer (G=1)        & 58.50 $\pm$ 4.72                             & 80.50 $\pm$ 7.53                            & 98.50 $\pm$ 0.50                           \\ 
		     \bottomrule
		    
	        \end{tabular}
               }
            \end{table}


\begin{figure}
     \centering
     \begin{subfigure}[b]{0.19\textwidth}
         \centering
         \includegraphics[width=\textwidth]{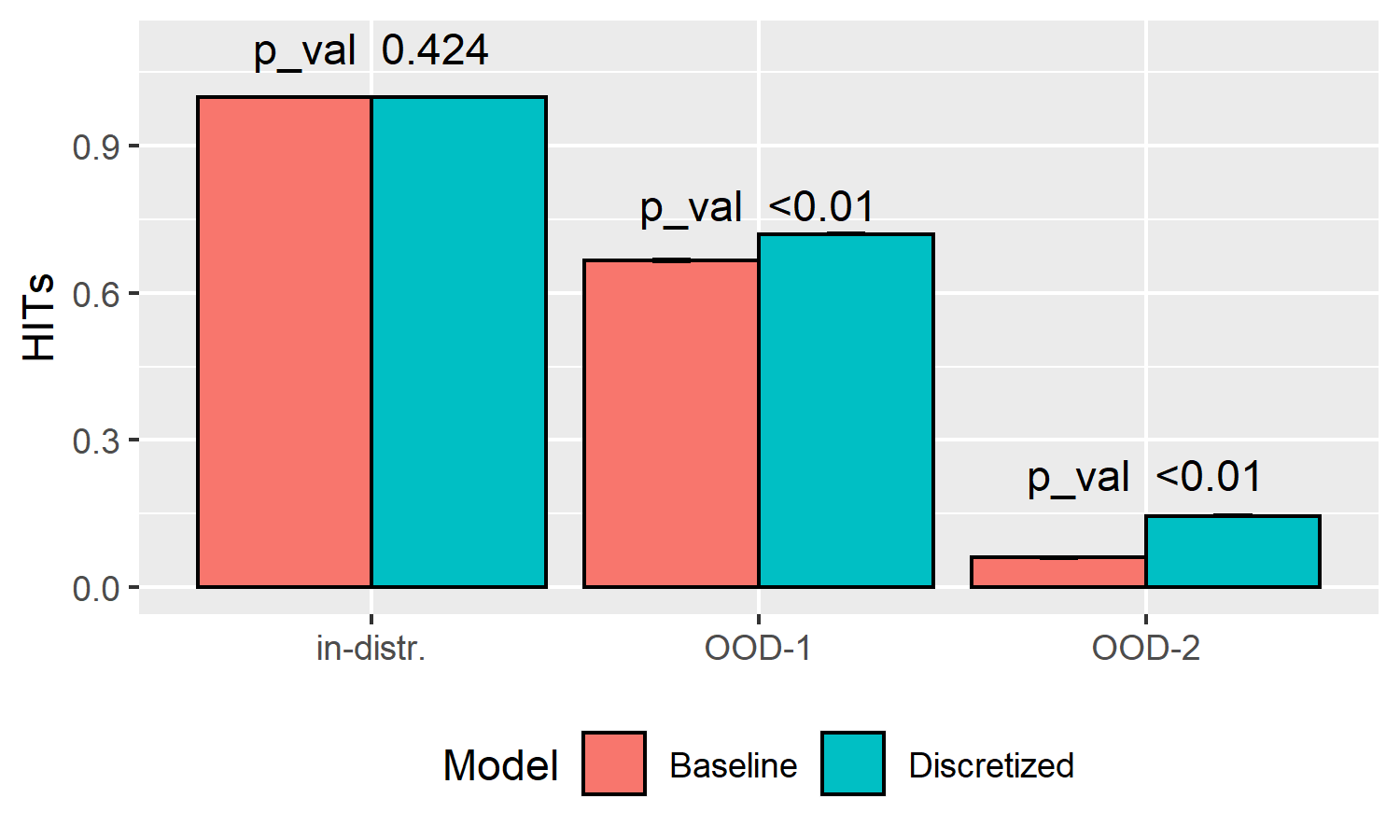}
         \caption{2D Shapes}
 \end{subfigure}
     \hfill
     \begin{subfigure}[b]{0.19\textwidth}
         \centering
         \includegraphics[width=\textwidth]{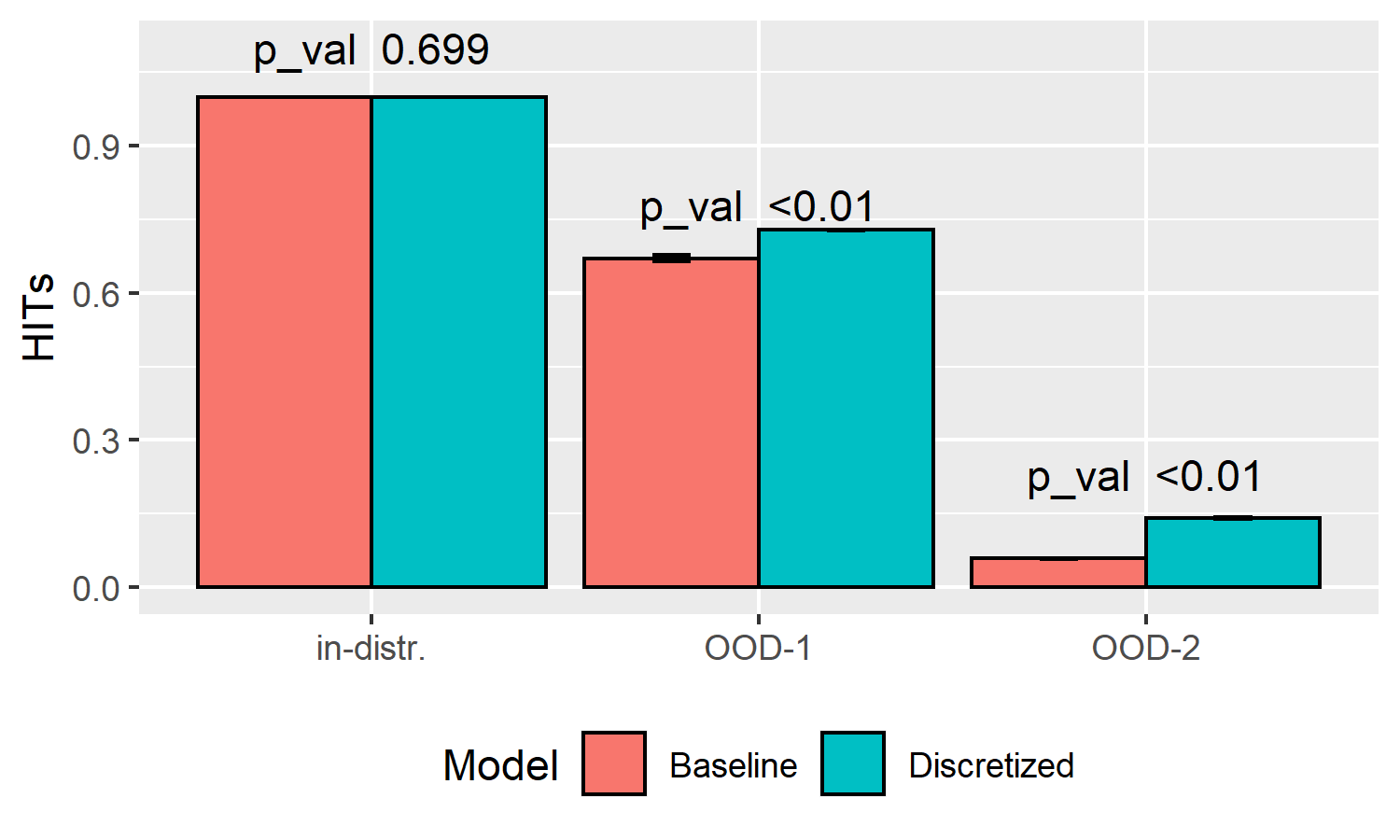}
         \caption{3D Shapes}
     \end{subfigure}
     \hfill
     \begin{subfigure}[b]{0.19\textwidth}
         \centering
         \includegraphics[width=\textwidth,trim={0 0 0 2cm},clip]{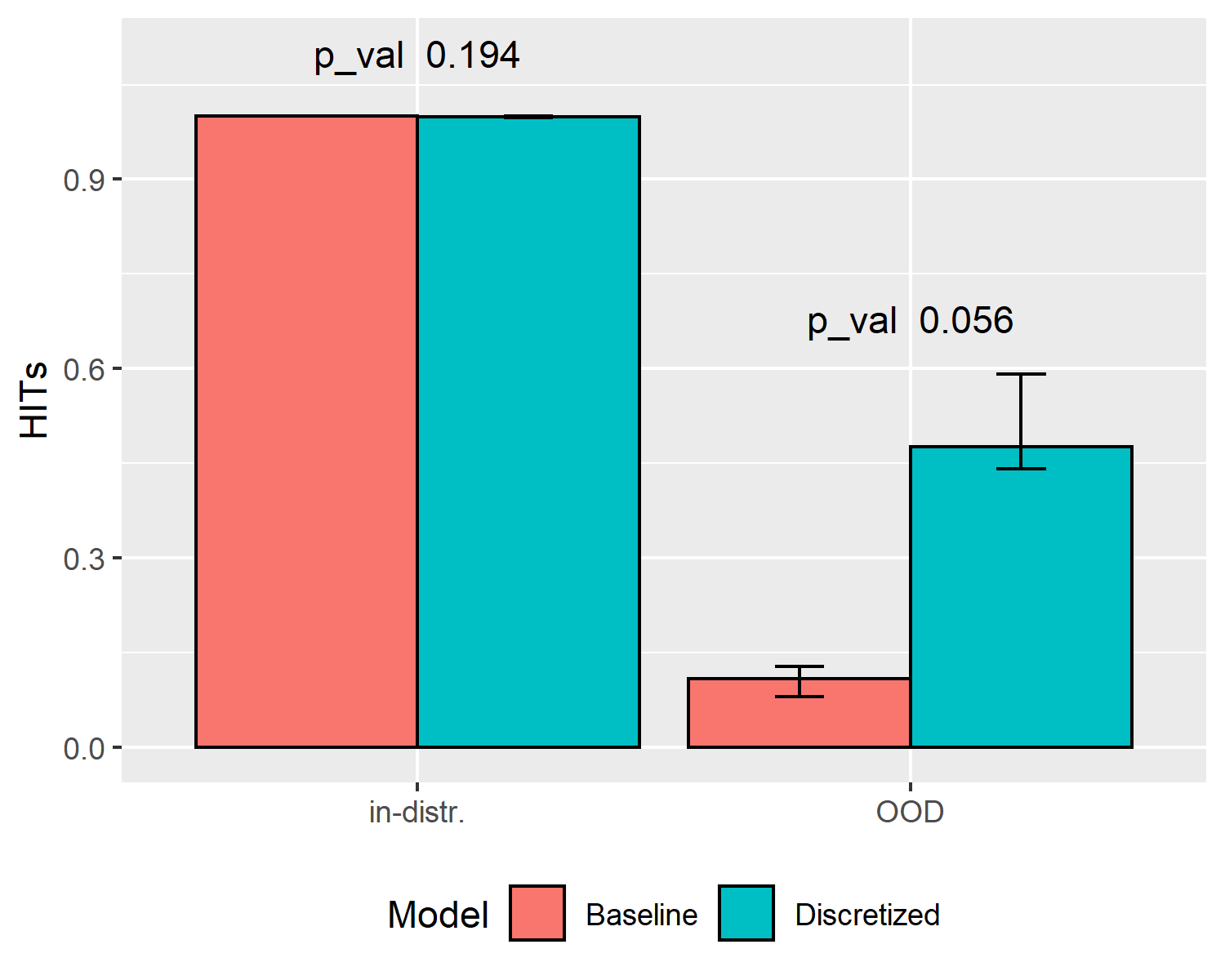}
         \caption{Three-Body}
     \end{subfigure}
     \hfill
     \begin{subfigure}[b]{0.19\textwidth}
         \centering
         \includegraphics[width=\textwidth]{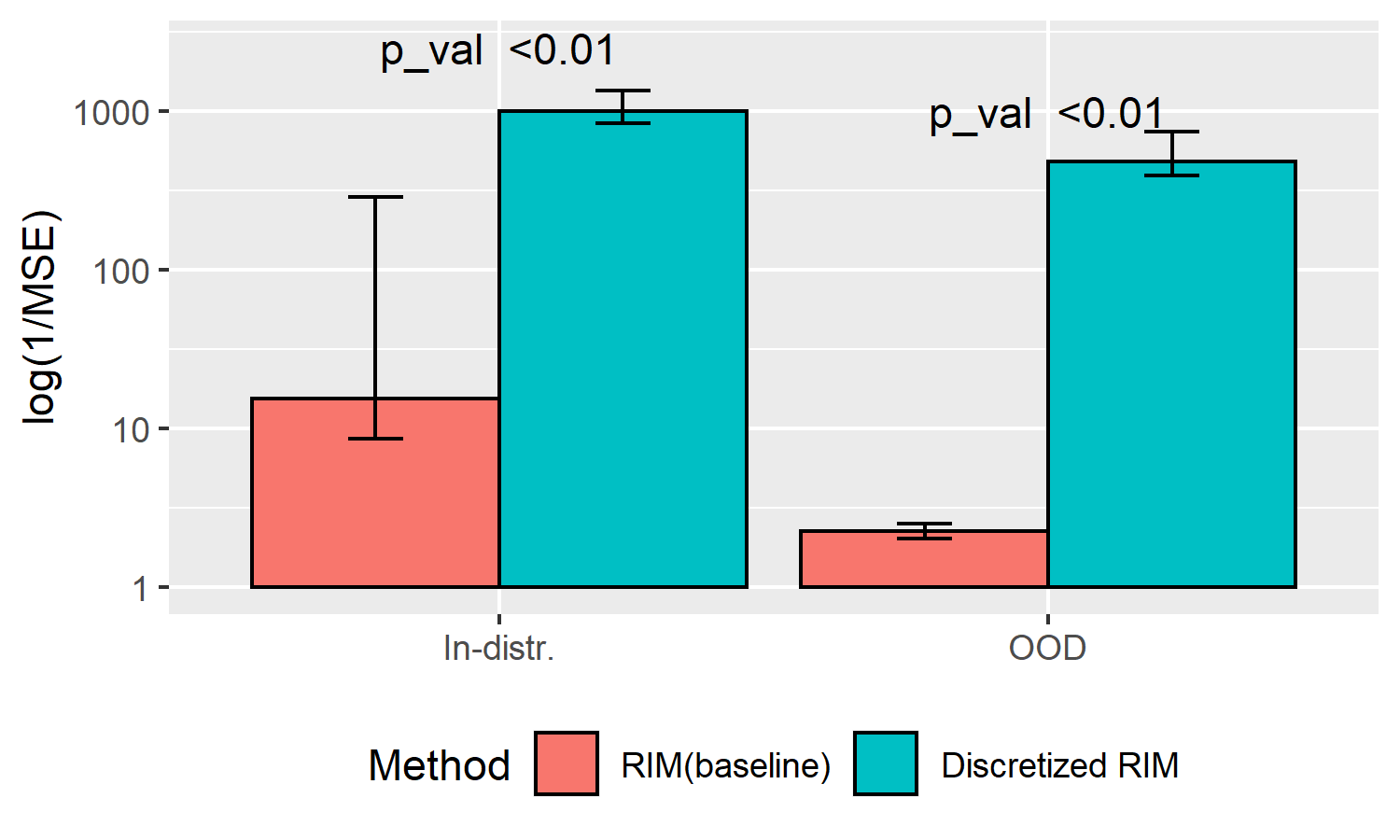}
         \caption{Adding}
     \end{subfigure}
     \hfill
     \begin{subfigure}[b]{0.19\textwidth}
         \centering
         \includegraphics[width=\textwidth]{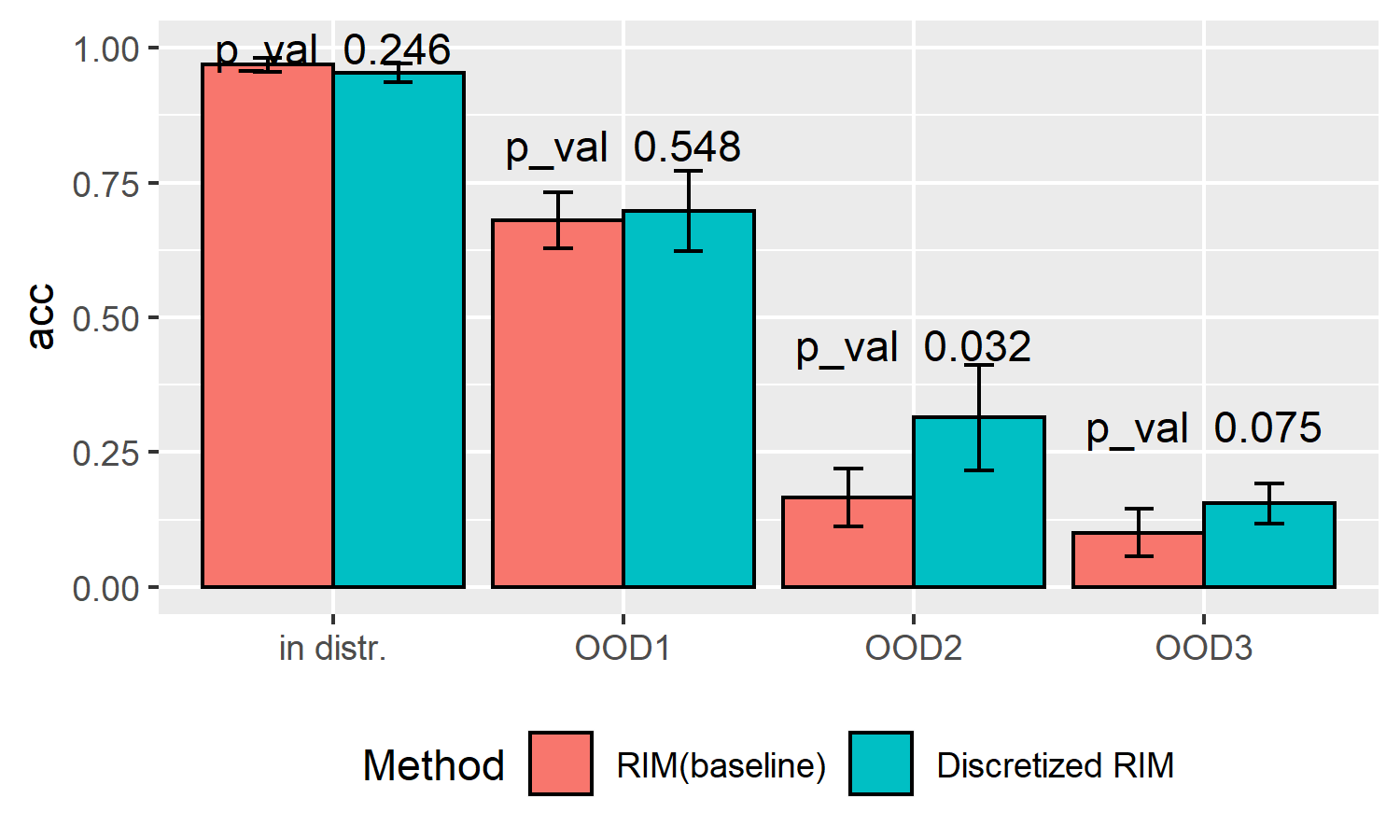}
         \caption{Seq. MNIST}
     \end{subfigure}

        \caption{GNN models (a,b and c) and RIMs with DVNC (d and e) show improved OOD generalization.  Test data in OOD 1 to 3 are increasingly different from training distribution.}
\label{fig:OODResult}
\end{figure}

\section{Experiments} \label{sec:experiments}


In this study we have two hypothesis: 1) The use of discrete symbols limits the bandwidth of communication. 2) The use of shared discrete symbols will promote systematic generalization.Theoretical results obtained in section \ref{sec:theory}, agree with hypothesis 1. In order to verify hypothesis 2, in this section, we design and conduct empirical experiments to show that discretization of inter-component communication improves OOD generalization and model performance.

\subsection{Communication along Edges of Graph Neural Networks}

We first turn our attention to visual reasoning task using GNNs and show that GNNs with DVNC have improved OOD generalization. In the tasks, the sequences of image frames in test set have different distribution from data used for training. We set up the model in a identical manner as in Kipf et al. 2019 \citep{kipf2019contrastive} except the introduction of DVNC. A CNN is used to to extract different objects from images. Objects are represented as nodes and relations between pairs of objects are represented as edges in the graph. The changes of nodes and edges can be learned by message passing in GNN. The information passed by all edges connected to a node is discretized by DVNC (see details in Section \ref{sec:GNN}). 

The following principles are followed when choosing and designing GNN visual reasoning tasks as well as all other OOD generalization tasks in subsequent sections: 1) Communication among different components are important for the tasks. 2) meaning of information communicated among components can be described in a discrete manner. 3) In OOD setting,distributions of information communicated among specialist components in test data also deviate from that in training data.

\paragraph{Object Movement Prediction in Grid World}

We begin with the tasks to predict object movement in grid world environments. Objects in the environments are interacting with each other and their movements are manipulated by actions applied to each object that give a push to the object in a randomly selected direction . Machine learning models are tasked to predict next positions of all objects in the next time step $t+1$ given their position and actions in time step $t$. We can think of this task as detecting whether an object can be pushed toward a direction. If object is blocked by the environment or another object in certain direction, then it can not be moved in that direction. Positions of objects are captures by nodes in GNN and the relative positions among different objects can be communicated in a discrete manner via message passing through edges.

We adapted and modified the original 2D shapes and 3D shapes movement tasks from \citet{kipf2019contrastive} by introducing different number of objects in training or testing environment. In both 2D and 3D shapes tasks ,five objects are available in training data, three objects are available in OOD-1 and only two objects are available in OOD-2. Our experimental results suggest that DVNC in GNN improved OOD generalization in a statistically significant manner (Figure~\ref{fig:OODResult}). In addition, the improvement is robust across different hyperparameters $G$ and $L$ (Figure~\ref{fig:DifferentGL}). Details of the visual reasoning tasks set up and model hyperparameters can be found in Appendix.

\paragraph{Three-body Physics Object Movement Prediction}
Next, we turn our attention to  a three-body-physics environment in which three balls  interacting with each other and move according to physical laws in classic mechanics in a 2D environment. There are no external actions applied on the objects. We  adapted and modified the three-body-physics environment from \citep{kipf2019contrastive}. In OOD experiment, the sizes of the balls are different from the original training data. Details of the experimental set up can be found in Appendix. Our experimental results show that GNN with DVNC improved OOD generalization (Figure~\ref{fig:OODResult} ).

\paragraph{Movement Prediction in Atari Games}
Similarly, we design OOD movement prediction tasks for 8 Atari games. Changes of each image frame depends on previous image frame and actions applied upon different objects. A different starting frame number is used to make the testing set OOD. GNN with DVNC showed statistically significant improvement in 6 out of 8 games and marginally significant improvement in the other games (Table~\ref{table_atari}).

\begin{table}
\centering
\caption{Graph Neural Networks benefit from discretized communication on OOD generalization in predicting movement in Atari games.  }
\resizebox{0.99\linewidth}{!}{
\begin{tabular}{c c c c | c c c c}
\hline
\textbf{Game} & \textbf{GNN (Baseline)}  & \textbf{GNN (Discretized)} & \textbf{P-Value (vs. baseline)} & \textbf{Game} & \textbf{GNN (Baseline)}  & \textbf{GNN (Discretized)} & \textbf{P-Value (vs. baseline)} \\ \midrule
Alien  & 0.1991 $\pm$ 0.0786 & \highlight{0.2876 $\pm$ 0.0782} & 0.00019  & DoubleDunk & 0.8680 $\pm$ 0.0281 & \highlight{0.8793 $\pm$ 0.0243}  & 0.04444 \\ 
BankHeist & 0.8224 $\pm$ 0.0323 & \highlight{0.8459 $\pm$ 0.0490} & 0.00002  & MsPacman & 0.2005 $\pm$ 0.0362 & \highlight{0.2325 $\pm$ 0.0648} & 0.05220\\ 
Berzerk & 0.6077 $\pm$ 0.0472 & \highlight{0.6233 $\pm$ 0.0509} & 0.06628 & Pong & 0.1440 $\pm$ 0.0845 & \highlight{0.2965 $\pm$ 0.1131} & 0.00041 \\ 
Boxing & 0.9228 $\pm$ 0.0806 & \highlight{0.9502 $\pm$ 0.0314} & 0.69409 & SpaceInvaders & 0.0460 $\pm$ 0.0225 & \highlight{0.0820 $\pm$ 0.0239} & 0.00960\\ \bottomrule
\end{tabular}}
\label{table_atari}
\end{table}

\subsection{Communication Between Positions in Transformers}

In transformer models, attention mechanism is used to communicate information among different position. We design and conduct two visual reasoning tasks to understand if discretizing results of attention in transformer models will help improve the performance (Section \ref{sec:Transformer}). In the tasks, transformers take sequence of pixel as input and detect relations among different objects which are communicated through attention mechanisms.

We experimented with the Sort-of-CLEVR visual relational reasoning task,  where the model is tasked with answering questions about certain properties of various objects and their relations with other objects \citep{santoro2017simple}. Each image in Sort-of-CLEVR contains  randomly placed geometrical shapes of different colors and  shapes. Each image comes with 10 relational questions and 10 non-relational questions. Nonrelational questions only consider properties of individual objects. On the other hand, relational questions consider relations among multiple objects. The input to the model consists of the image and the corresponding question. Each image and question
come with a finite number of possible answers and hence this task is to classify and pick the correct the answer\citep{goyal2021coordination}.
Transformer models with DVNC show significant improvement (Table~\ref{Table:TransformerResults}).



\subsection{Communication with Modular Recurrent Neural Networks}

Recurrent Independent Mechanisms(RIMs) are RNN with modular structures. In RIMs, units communicate with each other using attention mechanisms at each time step. We discretize results of inter-unit attention in RIMs (Section \ref{sec:RIM}). We conducted a numeric reasoning task and a visual reasoning task to understand if DVNC applied to RIMs improves OOD generalization. 

We considered a synthetic adding task in which the model is trained to compute the sum of a sequence of numbers followed by certain number of dummy gap tokens\citep{goyal2019recurrent}. In OOD settings of the task, the number of gap tokens after the target sequence is different in test set from training data. Our results show that DVNC makes significnat improvement in the OOD task(Figure~\ref{fig:OODResult}). 

For further evidence that RIMs with DVNC can better achieve OOD generlization, we consider the task of classifying MNIST digits as sequences of pixels \citep{krueger2016zoneout} and assay generalization to images of resolutions different from those seen during training. Our results suggest that RIMs with DVNC have moderately better OOD generalization than the baseline especially when the test images are very different from original training data (Figure~\ref{fig:OODResult}).

\begin{figure}[t!]
     \centering
     \begin{subfigure}[b]{0.24\textwidth}
         \centering
          \caption{2D Shapes (G)}
         \includegraphics[width=\textwidth]{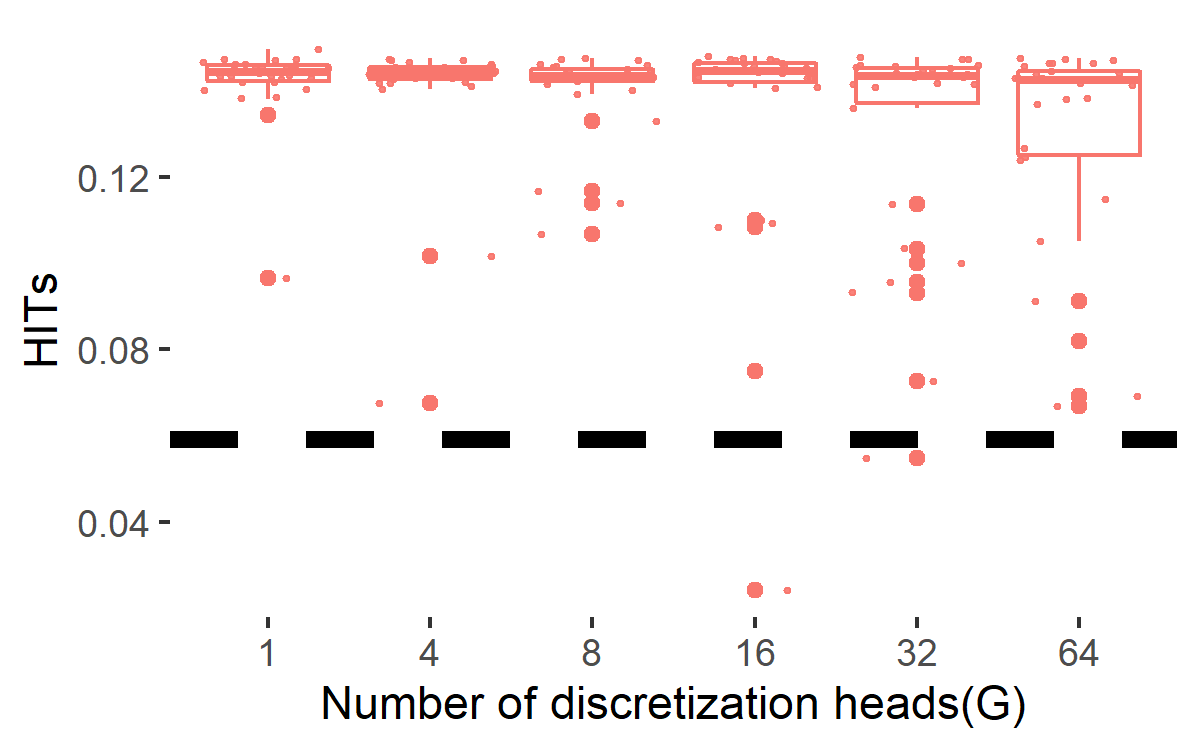}
        
     \end{subfigure}
     \hfill
     \begin{subfigure}[b]{0.24\textwidth}
         \centering
          \caption{2D Shapes (L)}
         \includegraphics[width=\textwidth]{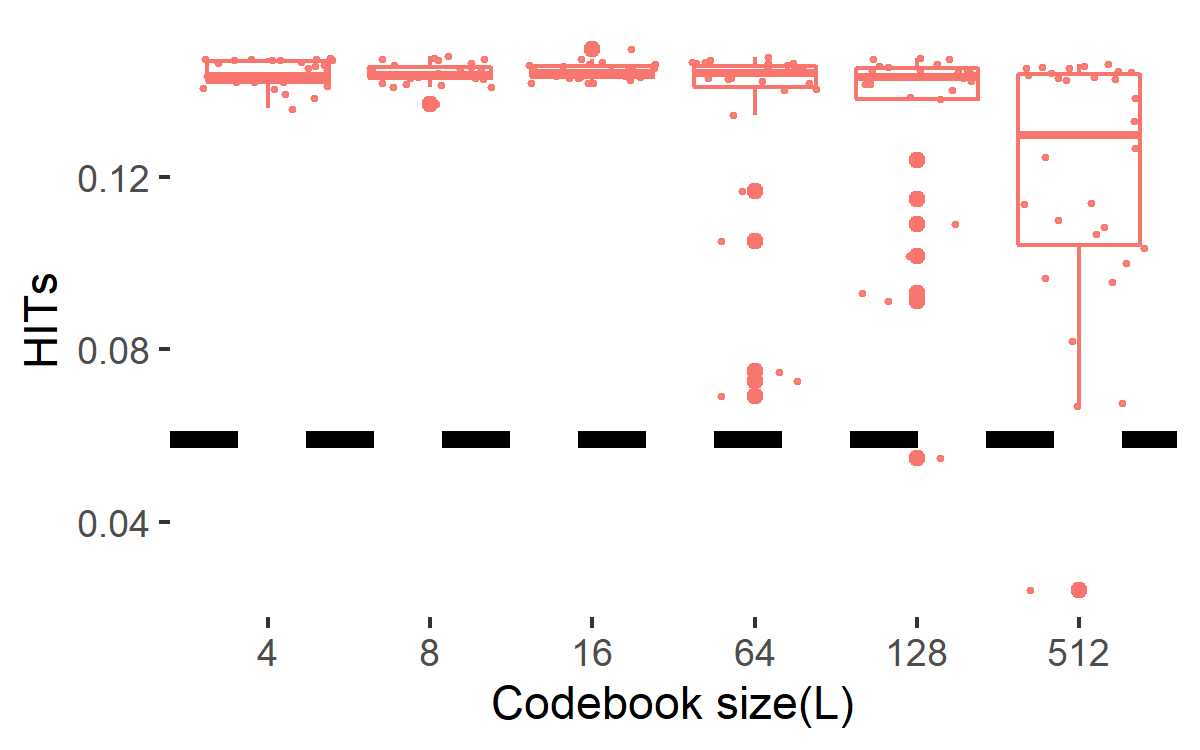}
        
     \end{subfigure}
     \hfill
     \begin{subfigure}[b]{0.24\textwidth}
         \centering
         \caption{3D Shapes (G)}
         \includegraphics[width=\textwidth]{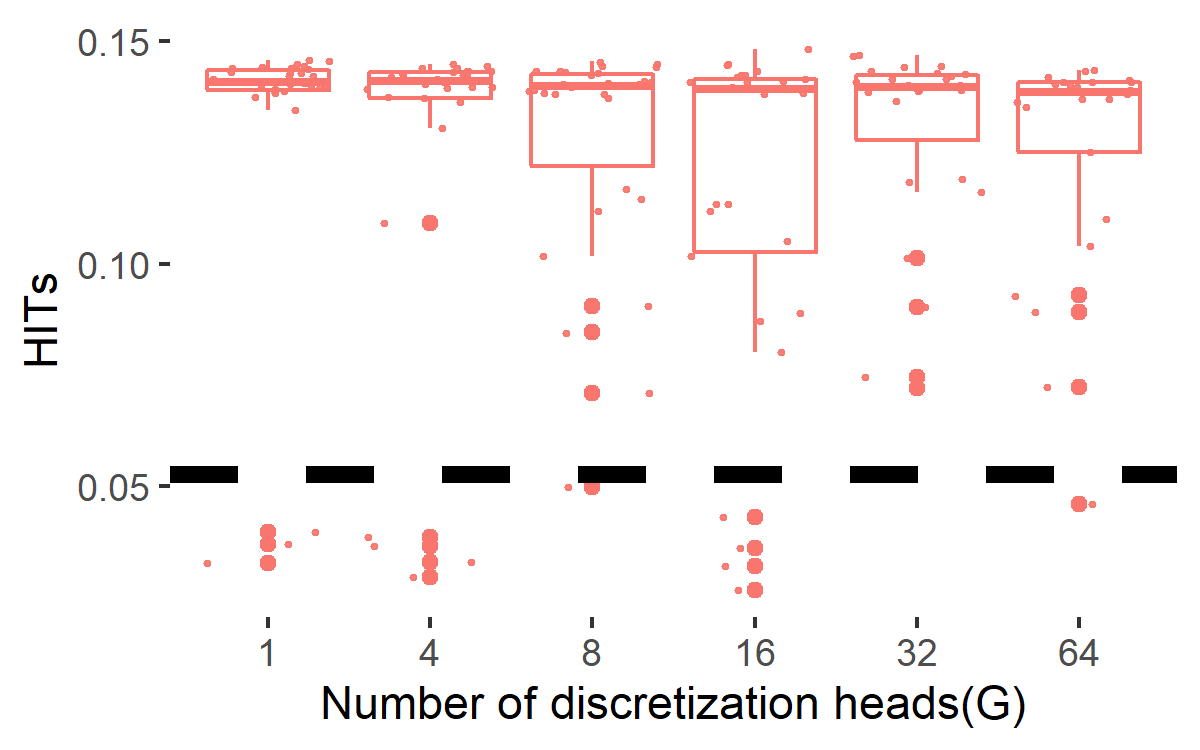}
         
     \end{subfigure}
      \hfill
     \begin{subfigure}[b]{0.24\textwidth}
         \centering
          \caption{3D Shapes (L)}
         \includegraphics[width=\textwidth]{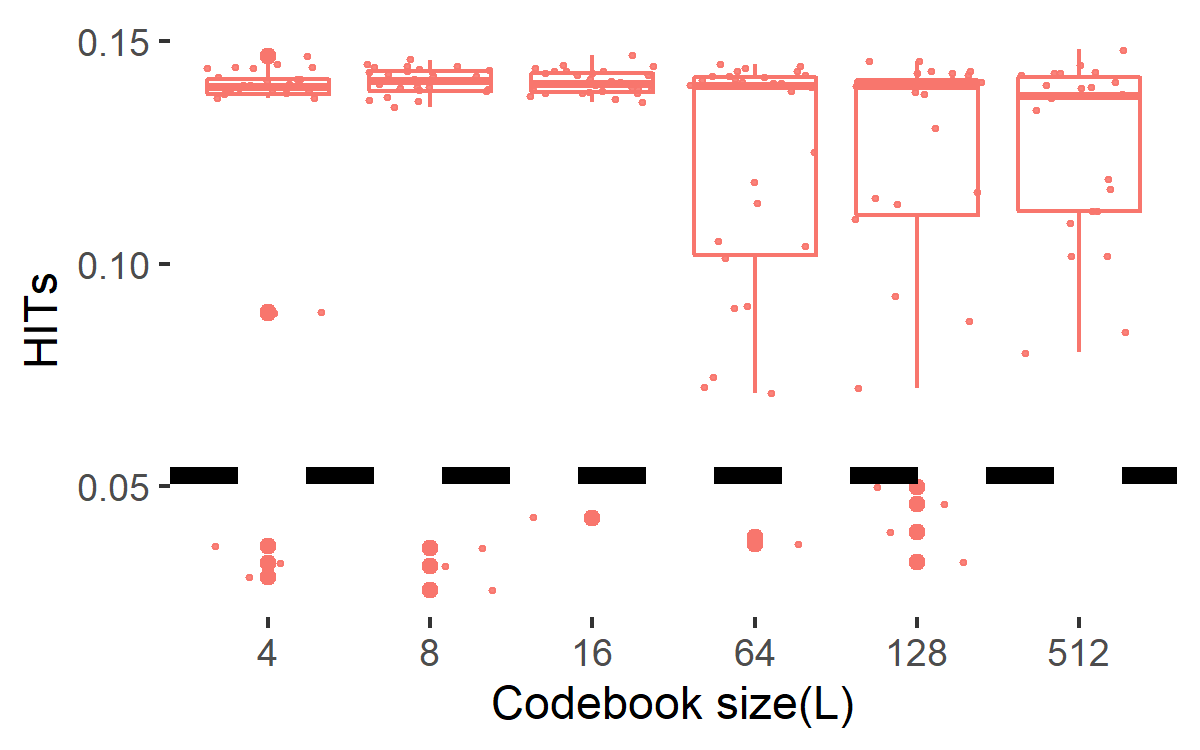}
        
     \end{subfigure}
     
     \begin{subfigure}[b]{0.23\textwidth}
         \centering
         \vspace{4mm}
         \caption{2D Shapes(GNN)}
         \vspace{-2mm}
         \includegraphics[width=\textwidth]{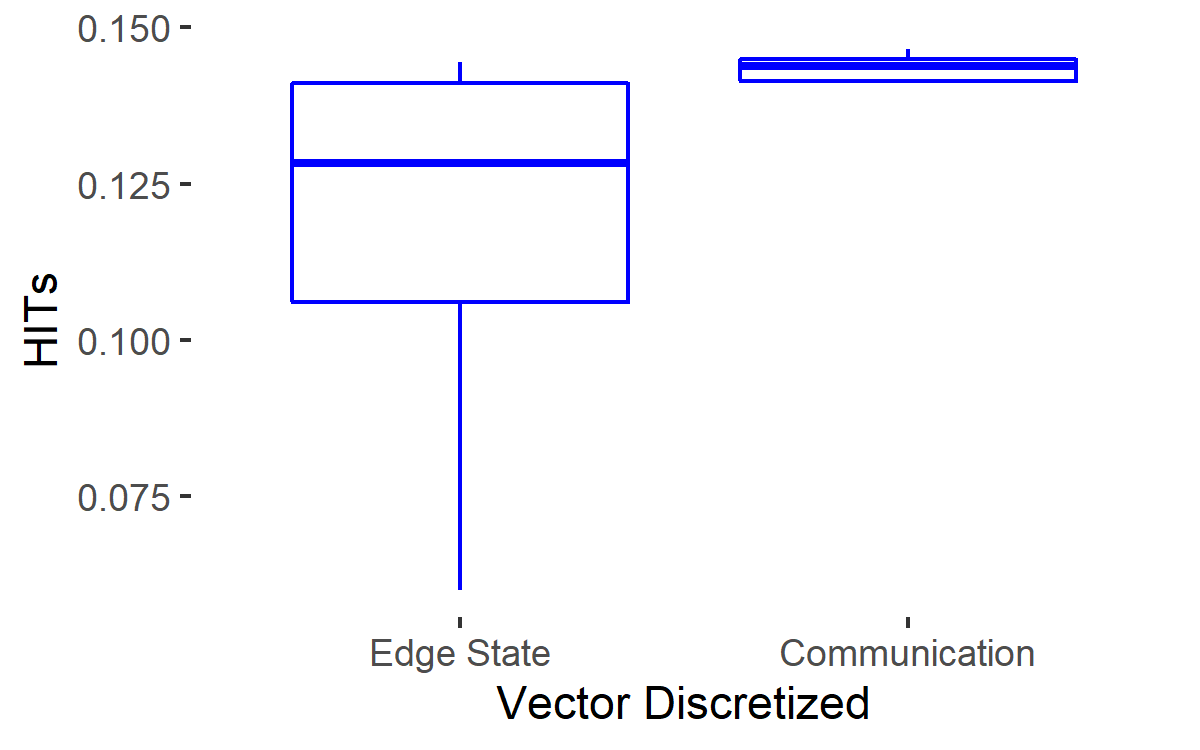}
     \end{subfigure}
     \hfill
     \begin{subfigure}[b]{0.23\textwidth}
         \centering
         \vspace{4mm}
         \caption{3D Shapes(GNN)}
         \vspace{-2mm}
         \includegraphics[width=\textwidth]{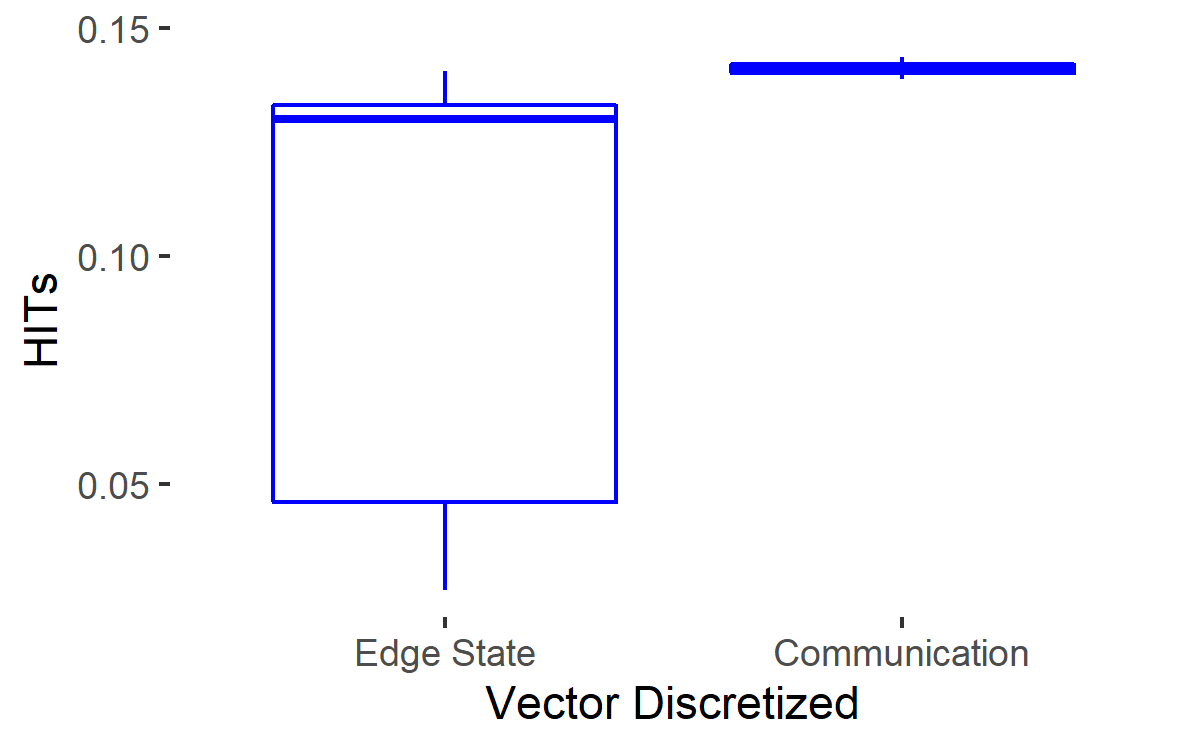}
         
     \end{subfigure}
    \hfill
     \begin{subfigure}[b]{0.48\textwidth}
         \centering
         \vspace{4mm}
         \caption{Adding Task(RIM)}
         \vspace{-2mm}
         \includegraphics[width=\textwidth,height=2cm]{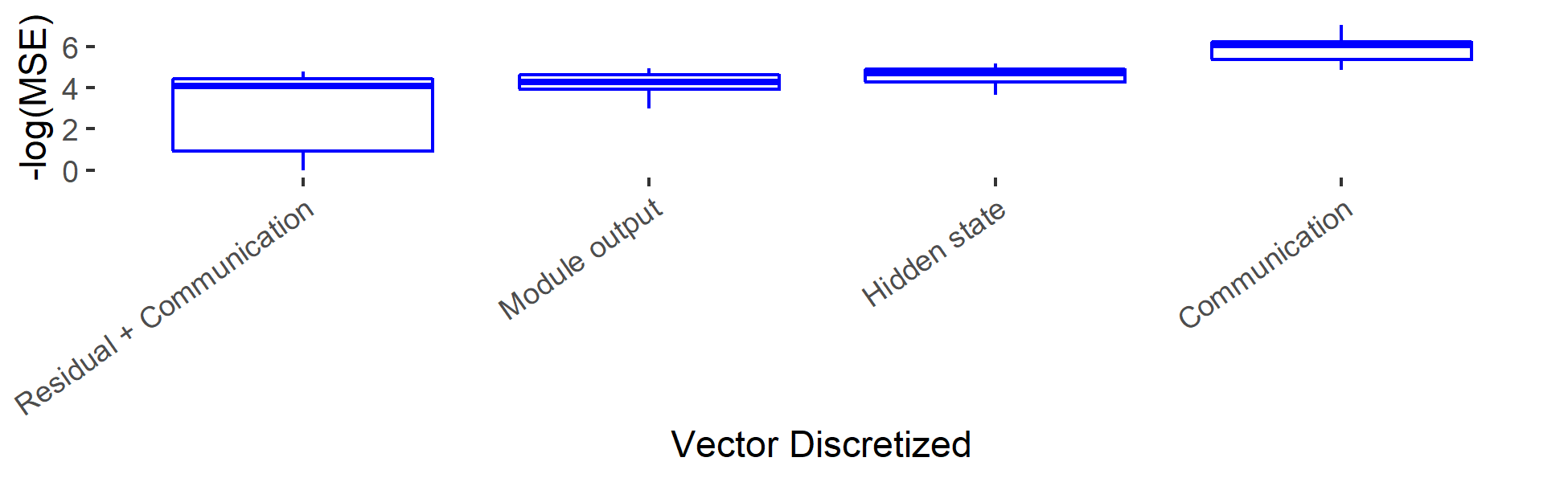}
        
     \end{subfigure}
     
\caption{Upper Row: Models with DVNC have improved OOD generalization in wide range of hyperparameter settings. Red dots are performances of models with DVNC with different codebook size (L) and number of heads (G). Black dashed lines are performance of baseline methods without discretization. Lower row: (e) and (f) compare OOD generalization in HITs (higher is better) between GNNs with results of communication discretized vs. edge states discretized. (g) Compares RIMs model with results of communication discretized vs. other vectors discretized. }
\label{fig:DifferentGL} 
\end{figure}

\subsection{Analysis and Ablations}

\paragraph{Discretizing Communication Results is Better than Discretizing other Parts of the Model: }  The intuition is that discretizing the results of communication with a shared codebook encourages more reliable and independent processing by different specialists.  We experimentally tested this key hypothesis on the source of the DVNC's success by experimenting with discretizing other parts of the model.  For example, in RIMs, we tried discretizing the updates to the recurrent state and tried discretizing the inputs.  On the adding task, this led to improved results over the baseline, but performed much worse than discretizing communication.  For GNNs we tried discretizing the input to the communication (the edge hidden states) instead of the result of communication, and found that it led to significantly worse results and had very high variance between different trials.  These results are in Figure~\ref{fig:DifferentGL}.  

\paragraph{VQ-VAE Discretization Outperforms Gumbel-Softmax: } The main goal of our work was to demonstrate the benefits of communication with discrete values, and this discretization could potentially be done through a number of different mechanisms.  Experimentally, we found that the nearest-neighbor and straight-through estimator based discretization technique,similar to method used in VQ-VAE, outperformed the use of a Gumbel-Softmax to select the discrete tokens (Figure~\ref{fig:gumbel} in Appendix).  These positive results led us to focus on the VQ-VAE discretization technique, but in principle other mechanisms could also be recruited to accomplish our discrete-valued neural communication framework. We envision that DVCN should work complementarily with future advances in learning discrete representations.

\vspace{-2mm}
\section{Discussion}
\vspace{-2mm}
With the evolution of deep architectures from the monolithic MLP to complex architectures with specialized components, we are faced with the issue of how to ensure that subsystems can communicate and coordinate.
Communication via continuous, high-dimensional signals is a natural choice given the history of deep learning, but our work argues that discretized communication results in more robust, generalizable learning. Discrete-Valued Neural Communication (DVNC) achieves a much lower noise-sensitivity bound while allowing high expressiveness through the use of multiple discretization heads.  This technique is simple and easy-to-use in practice and improves out-of-distribution generalization. DVNC is applicable to all structured architectures that we have examined where inter-component communication is important  and the information to be communicated can be discretized by its nature. 
\paragraph{Limitations} 
The proposed method has two major limitations. First, DVNC can only improve performance if communication among specialists is important for the task. If the different components do not have good specialization, then DVNC's motivation is less applicable. Another limitation is that the discretization process can reduce the expressivity of the function class, although this can be mitigated by using a large value for $G$ and $L$ and can be partially monitored by the quantity of training data (e.g., training loss) similarly to the principle of structural minimization. Hence future work could examine how to combine discrete communication and continuous communication. 
\paragraph{Social Impact} Research conducted in this study is purely technical. The authors expect no direct  negative nor positive social impact.

\clearpage
\typeout{} 
\bibliography{References}
\bibliographystyle{apalike}


\clearpage
\section*{Checklist}

The checklist follows the references.  Please
read the checklist guidelines carefully for information on how to answer these
questions.  For each question, change the default \answerTODO{} to \answerYes{},
\answerNo{}, or \answerNA{}.  You are strongly encouraged to include a {\bf
justification to your answer}, either by referencing the appropriate section of
your paper or providing a brief inline description.  For example:
\begin{itemize}
  \item Did you include the license to the code and datasets? \answerYes{See Section 2 (`General formatting instructions') of the tex template (`Formatting Instructions For NeurIPS 2021').}
  \item Did you include the license to the code and datasets? \answerNo{The code and the data are proprietary.}
  \item Did you include the license to the code and datasets? \answerNA{}
\end{itemize}
Please do not modify the questions and only use the provided macros for your
answers.  Note that the Checklist section does not count towards the page
limit.  In your paper, please delete this instructions block and only keep the
Checklist section heading above along with the questions/answers below.

\begin{enumerate}

\item For all authors...
\begin{enumerate}
  \item Do the main claims made in the abstract and introduction accurately reflect the paper's contributions and scope?
    \answerYes{Abstract and introduction directly reflect the paper's contribution and scope }

  \item Did you describe the limitations of your work?
     \answerYes{see last paragraph of section 6 of the manuscript }

  \item Did you discuss any potential negative societal impacts of your work?
    \answerNA{This study is purely technical and do not involve any social aspects}

  \item Have you read the ethics review guidelines and ensured that your paper conforms to them?
    \answerYes{}
\end{enumerate}

\item If you are including theoretical results...
\begin{enumerate}
  \item Did you state the full set of assumptions of all theoretical results?
    \answerYes{See section 2 and related appendix}
    
	\item Did you include complete proofs of all theoretical results?
    \answerYes{see appendix}
\end{enumerate}

\item If you ran experiments...
\begin{enumerate}
  \item Did you include the code, data, and instructions needed to reproduce the main experimental results (either in the supplemental material or as a URL)?
    \answerYes{see section 2, URL available after double-blind review}

  \item Did you specify all the training details (e.g., data splits, hyperparameters, how they were chosen)?
    \answerYes{see appendix}

	\item Did you report error bars (e.g., with respect to the random seed after running experiments multiple times)?
    \answerYes{see figures with bar and box plot, error bars and p-val provided}
    
	\item Did you include the total amount of compute and the type of resources used (e.g., type of GPUs, internal cluster, or cloud provider)?
    \answerYes{See appendix}
\end{enumerate}

\item If you are using existing assets (e.g., code, data, models) or curating/releasing new assets...
\begin{enumerate}
  \item If your work uses existing assets, did you cite the creators?
    \answerYes{ all cited}

  \item Did you mention the license of the assets?
    \answerYes{}

  \item Did you include any new assets either in the supplemental material or as a URL?
    \answerYes{ new dataset released together with code on github after double-blind review}

  \item Did you discuss whether and how consent was obtained from people whose data you're using/curating?
    \answerNo{ publicly available data used}

  \item Did you discuss whether the data you are using/curating contains personally identifiable information or offensive content?
    \answerNA{ no personally identifiable information included}
\end{enumerate}

\item If you used crowdsourcing or conducted research with human subjects...
\begin{enumerate}
  \item Did you include the full text of instructions given to participants and screenshots, if applicable?
    \answerNA{}
  \item Did you describe any potential participant risks, with links to Institutional Review Board (IRB) approvals, if applicable?
    \answerNA{}
  \item Did you include the estimated hourly wage paid to participants and the total amount spent on participant compensation?
    \answerNA{}
\end{enumerate}

\end{enumerate}


\clearpage

\appendix

\allowdisplaybreaks

\clearpage

\section{Additional theorems for theoretical motivations} \label{sec:2}
In this appendix, as a complementary to Theorems \ref{prop:1}--\ref{prop:2}, we provide additional theorems, Theorems \ref{thm:1}--\ref{thm:2}, which  further illustrate the two advantages  of the discretization process by considering an   abstract model with the discretization bottleneck.
For the  advantage on the sensitivity, the error due to potential noise and perturbation without discretization --- the third term  $\xi(w,r',\Mcal',d)>0$ in Theorem \ref{thm:2} --- is shown to be minimized to  zero with discretization  in Theorems \ref{thm:1}. For the second advantage,       the  underlying dimensionality of  $\Ncal_{(\Mcal',d')}(r',\Hcal)+\ln(\Ncal_{(\Mcal,d)}(r, \Theta)/\delta)$ without discretization (in the bound of Theorem \ref{thm:2}) is proven to be reduced to the typically much smaller  underlying dimensionality of $L^{G}+ \ln(\Ncal_{(\Mcal,d)}(r, E\times \Theta)$ with discretization in  Theorems \ref{thm:1}. Here, for any metric space $(\Mcal,d)$ and subset $M\subseteq \Mcal$,   the  $r$-converging number of $M$\ is defined by 
$
\Ncal_{(\Mcal,d)}(r,M)=\min\left\{|\Ccal|: \Ccal \subseteq \Mcal, M\subseteq \cup_{c \in\Ccal} \Bcal_{(\Mcal,d)}[c,r]\right.\}$ where the (closed) ball of radius $r$ at centered at $c$ is denoted by $\Bcal_{(\Mcal,d)}[c,r]=\{x\in \Mcal : d(x,c)\le r\}$.  See Appendix \ref{app:2} for a   simple   comparison between the bound of Theorem \ref{thm:1} and that of Theorem \ref{thm:2} when  the metric spaces $(\Mcal,d)$ and $(\Mcal',d')$ are chosen to be Euclidean spaces.

We now introduce the  notation used in Theorems \ref{thm:1}--\ref{thm:2}. Let $q_e(h):=q(h,L,G)$. The models are defined by $\tf(x):=\tf(x,w,\theta):=
(\varphi_w \circ h_\theta)(x)$ without the discretization and   $f(x):=f(x,w,e,\theta) :=(\varphi_w \circ q_e \circ h_\theta)(x)$ with the discretization.
 Here, $\varphi_w$ represents a deep neural network  with  weight parameters $w\in \Wcal \subset \RR^{D}$, $q_e$ is the discretization process with the  codebook $e \in E  \subset \RR^{L\times m}$, and $h_\theta$ represents a deep neural network with  parameters $\theta\in \Theta \subset\RR^{\zeta}$. Thus, the tuple of all learnable parameters are $(w, e, \theta)$. For the codebook space,  $E=E_1 \times E_2$ with $E_1 \subset \RR^L$ and $E_2 \subset \RR^m$. Moreover, let $J:(f(x) ,y) \mapsto J(f(x) ,y)\in \RR $ be an arbitrary (fixed) function,   $h_{\theta}(x) \in \Hcal \subset \RR^m$,
  $x \in \Xcal$, and $y \in \Ycal =\{y^{(1)}, y^{(2)}\}$ for some $y^{(1)}$ and $y^{(2)}$.

\begin{theorem} \label{thm:1}
\emph{(with discretization)}
Let $C_{ J}(w)$ be the smallest real number such that $| J(\varphi_w(\eta) ,y)|\le C_{ J}(w)$ for all $(\eta,y)\in E_2\times \Ycal $. Let $\rho \in \NN^+$ and $(\Mcal,d)$ be a matric space such that $E\times \Theta \subseteq \Mcal$. Then, for any $\delta>0$,  with  probability at least $1-\delta$ over an iid draw of $n$ examples $((x_i, y_i))_{i=1}^n$, the following holds: for any $(w,e,\theta) \in \Wcal \times E \times \Theta$,   
\begin{align*}
&\left|\EE_{x,y}[J(f(x,w,e,\theta) ,y)] - \frac{1}{n}\sum_{i=1}^nJ(f(x_{i},w,e,\theta) ,y_{i}) \right| 
\\ & \le C_{J}(w) \sqrt{\frac{4L ^{G}\ln2 + 2 \ln(\Ncal_{(\Mcal,d)}(r, E\times \Theta)/\delta)}{n}}+ \sqrt{\frac{\Lcal_{d} (w)^{2/\rho}}{n}}, 
\end{align*}
where $r=\Lcal_{d} (w)^{1/\rho-1} \sqrt{\frac{1}{n}}$ and $\Lcal_{d} (w)\ge 0$ is the smallest real number such that for all $(e,\theta)$ and $(e',\theta')$ in $E\times \Theta$,
$
|\psi_{w}(e,\theta) -\psi_{w}(e',\theta')| \le\Lcal_{d} (w) d((e,\theta), (e',\theta'))$ with $\psi_{w}(e,\theta)= \EE_{x,y}[ J(f(x) ,y)]-\frac{1}{n}\sum_{i=1}^n  J(f(x_{i}) ,y_{i})$
\end{theorem}

\begin{theorem} \label{thm:2}
\emph{(without discretization)} 
 Let $\tC_{ J}(w)$ be the smallest real number such that $| J((\varphi_w \circ h_\theta)(x) ,y)|\le \tC_{ J}(w)$ for all $(\theta,x,y)\in \Theta \times \Xcal \times \Ycal $. Let  $\rho \in \NN^+$ and $(\Mcal,d)$ be a matric space such that $\Theta \subseteq \Mcal$. Let $(\Mcal',d')$ be a matric space such that $\Hcal \subseteq \Mcal'$. Fix $r'> 0$ and  $\bCcal_{r',d'} \in \argmin_{\Ccal}\{|\Ccal|:\Ccal \subseteq\Mcal',\Hcal \subseteq \cup_{c \in\Ccal} \Bcal_{(\Mcal', d')}[c,r ']\}$. Assume that for any $c \in\bCcal_{r',d'}$, we have $|(J(\varphi_w(h) ,y)-(J(\varphi_w(h') ,y)|\le \xi(w,r',\Mcal',d)$  for any $h,h'  \in \Bcal_{(\Mcal', d')}[c,r ']$ and $y \in \Ycal$. Then, for any $\delta>0$,  with  probability at least $1-\delta$ over an iid draw of $n$ examples $((x_i, y_i))_{i=1}^n$, the following holds: for any $(w,\theta) \in \Wcal \times \Theta$,   
\begin{align*}
&\left|\EE_{x,y}[J(\tf(x,w,\theta) ,y)] - \frac{1}{n}\sum_{i=1}^nJ(\tf(x_{i},w,\theta) ,y_{i}) \right| 
\\ & \le \tC_{J}(w) \sqrt{\frac{4\Ncal_{(\Mcal',d')}(r',\Hcal) \ln2 + 2 \ln(\Ncal_{(\Mcal,d)}(r, \Theta)/\delta)}{n}}+ \sqrt{\frac{\tLcal_{d} (w)^{2/\rho}}{n}}+  \xi(w,r',\Mcal',d), 
\end{align*}
where $r=\tLcal_{d} (w)^{1/\rho-1} \sqrt{\frac{1}{n}}$ and $\tLcal_{d} (w)\ge 0$ is the smallest real number
such that for all $\theta$ and $\theta'$ in $ \Theta$,
$
|\tpsi_{w}(\theta) -\tpsi_{w}(\theta')| \le\tLcal_{d} (w) d(\theta, \theta')
$ with  $\tpsi_{w}(\theta)= \EE_{x,y}[ J(\tf(x) ,y)]-\frac{1}{n}\sum_{i=1}^n  J(\tf(x_{i}) ,y_{i})$.
 \end{theorem}
Note that we have  $C_{ J}(w)\le \tC_{ J}(w)$ and  $\Lcal_{d} (w) \approx \tLcal_{d}(w)$ by their definition.  For example, if we  set  $J$ to be a loss criterion, the  bound in Theorem \ref{thm:2}  becomes in the same order as and comparable to the  generalization bound via the \textit{algorithmic robustness} approach proposed by the previous papers
\citep{xu2012robustness,sokolic2017generalization,sokolic2017robust}, as we show in  Appendix \ref{app:3}. 

\section{Additional Experiments}
\label{app:experiments}

\begin{figure}
    \centering
    \includegraphics[width=0.8\textwidth]{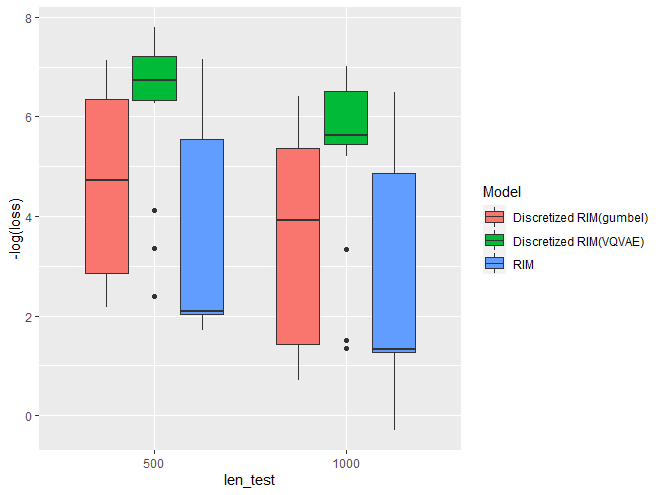}
    \caption{Performance on adding task (RIMs) with no discretization, Gumbel-Softmax discretization, or VQ-VAE style discretization (ours). Test length=500 is in-distribution test result and test length=1000 is out-of-distribution results.  }
    \label{fig:gumbel}
\end{figure}

\section{Additional discussions on theoretical motivations}

\subsection{Simple comparison of Theorems \ref{thm:1} and \ref{thm:2} with Euclidean space} \label{app:2}

For the purpose of  of  the  comparison, we will now consider the simple worst case with no additional structure with  the Euclidean space to instantiate   $\Ncal_{(\Mcal',d')}(r',\Hcal)$, $\Ncal_{(\Mcal,d)}(r, \Theta)$, and $\Ncal_{(\Mcal,d)}(r, E\times \Theta)$. It should be obvious that we can improve the bounds via considering metric spaces with additional structures. For example, we can consider a lower dimensional manifold $\Hcal$ in the ambient space of $\RR^m$ to reduce   $\Ncal_{(\Mcal',d')}(r',\Hcal)$. Similar ideas can be applied for $\Theta$ and $E\times \Theta$. Furthermore, the invariance as well as  margin  were used to reduce the bound on $\Ncal_{(\Mcal',d')}(r',\Xcal)$   in previous works \citep{sokolic2017generalization,sokolic2017robust} and similar ideas can be applied for $\Ncal_{(\Mcal',d')}(r',\Hcal)$, $\Ncal_{(\Mcal,d)}(r, \Theta)$, and $\Ncal_{(\Mcal,d)}(r, E\times \Theta)$. In this regard, the discretization  can be viewed as a method to minimize $\Ncal_{(\Mcal',d')}(r',\Hcal)$ to easily controllable  $L^G$ while minimizing the sensitivity term $ \xi(w,r',\Mcal',d)$ to zero at the same time in Theorems \ref{thm:1} and \ref{thm:2}.

Suppose that for any $y\in \Ycal$, the function $h \mapsto  J(\varphi_w(h) ,y)$ is Lipschitz continuous as  
$|( J(\varphi_w(h) ,y)-( J(\varphi_w(h') ,y)|\le \varsigma(w)d(h,h')$. Then, we can set $\xi(w,r',\Mcal',d)= 2\varsigma(w)r'$ since $d(h,h')\le2 r'$ for any $h,h'  \in \Bcal_{(\Mcal', d')}[c,r ']$.

As an simple example, let us choose the metric space $(\Mcal',d')$ to be the Euclidean space $\RR^{m}$ with the Euclidean metric and $\Hcal \subset \RR^{m}$  such that $\|v\|_2 \le R_{\Hcal}$ for all $v \in \Hcal$. Then, we have  $ \Ncal_{(\Mcal',d')}(r',\Hcal)\le (2R_{\Hcal} \sqrt{m}/r')^m$ and  we can set $\xi(w,r',\Mcal',d)=2\varsigma(w)r'$. Thus, by setting $r'=R_{\Hcal}/2$, we can replace  $ \Ncal_{(\Mcal',d')}(r',\Hcal)$  by $(4 \sqrt{m})^m$ and  set $\xi(w,r',\Mcal',d)=\varsigma(w)R_{\Hcal}$.

Similarly, let us choose the metric space $(\Mcal,d)$ to be the Euclidean space with the Euclidean metric and $E \subset  \RR^{Lm}$  and $\Theta \subset\RR^{\zeta}$ such that $\|v\|_2 \le R_{E}$ for all $v \in E$ and  $\|v\|_2 \le R_{\Theta }$ for all $v \in \Theta$. This implies that  $\|(v_E,v_\theta)\|_{2} \le \sqrt{R_{E}^2 + R_{\Theta}^2}$. Thus, we have  $ \Ncal_{(\Mcal,d)}(r,  \Theta)\le(2R_{\Theta }  \sqrt{\zeta}/r)^\zeta$ and $\Ncal_{(\Mcal,d)}(r, E\times \Theta)\le(2\sqrt{R_{E}^2 + R_{\Theta}^2}  \sqrt{Lm+\zeta}/r)^{Lm+\zeta}$. Since  $r=\tLcal_{d} (w)^{1/\rho-1} \sqrt{\frac{1}{n}}$  and   $r=\Lcal_{d} (w)^{1/\rho-1} \sqrt{\frac{1}{n}}$, we can replace  $ \Ncal_{(\Mcal,d)}(r,  \Theta)$  by $(2R_{\Theta }\tLcal_{d} (w)^{1-1/\rho}  \sqrt{\zeta n})^\zeta$ and $\Ncal_{(\Mcal,d)}(r, E\times \Theta)$ by  $(2\Lcal_{d} (w)^{1-1/\rho}\sqrt{R_{E}^2 + R_{\Theta}^2}  \sqrt{(Lm+\zeta)n})^{Lm+\zeta}$. 
By summarizing these and ignoring the logarithmic dependency as in the standard $\tilde O$ notation, we have the following bounds for Theorems \ref{thm:1} and \ref{thm:2}:
$$
\text{(with discretization)} \quad C_{ J}(w) \sqrt{\frac{4L^{G}  + 2Lm+2\zeta+2\ln(1/\delta)}{n}}+ \sqrt{\frac{\Lcal_{d} (w)^{2/\rho}}{n}},
$$
and
$$
\text{(without discretization)}  \quad \tC_{ J}(w) \sqrt{\frac{4(4 \sqrt{m})^m  + 2 \zeta+2\ln(1/\delta)}{n}}+ \sqrt{\frac{\tLcal_{d} (w)^{2/\rho}}{n}}+ \varsigma(w)R_{\Hcal},
$$
where we used the fact that $\ln (x/y)=\ln(x) + \ln (1/y)$. Here, we can more easily see that the discretization process has the benefits in the two aspects:
\begin{enumerate}
\item 
The discretization process improves sensitivity against noise and perturbations:  i.e., it reduces the sensitivity term  $ \varsigma(w)R_{\Hcal}$ to be zero.
\item
The discretization process reduces underlying dimensionality:   i.e., it reduce the term of $4(4 \sqrt{m})^m  $ to the term of $4L^{G}  + 2Lm$. In practice, we typically have $4(4 \sqrt{m})^m  \gg4L^{G}  + 2Lm$. This shows that using the discretization process withe codebook of size $L \times m$, we can successfully reduce the exponential dependency on $m$ to the linear dependency on $m$. This is a significant improvement.  
\end{enumerate}

\subsection{On the comparison of Theorem  \ref{thm:2} and algorithmic robustness} \label{app:3}
If we assume  that the function   $x\mapsto \ell(\tf(x),y)$ is Lipschitz for all $y \in \Ycal$ with Lipschitz constant $\varsigma_{x}(w)$ similarly to our assumption in Theorem \ref{thm:2}, the bound via the algorithmic robustness in the previous paper \citep{xu2012robustness} becomes the following: for any $\delta>0$,  with  probability at least $1-\delta$ over an iid draw of $n$ examples $((x_i, y_i))_{i=1}^n$, for any $(w,\theta) \in \Wcal \times \Theta$, 
\begin{align}\label{eq:15}
& \left| \EE_{x,y}[\ell(\tf(x,w,\theta),y)] - \frac{1}{n}\sum_{i=1}^n[\ell(\tf(x_{i},w,\theta),y_{i})] \right| 
\\ \nonumber & \le  \hat C_J  \sqrt \frac{4\Ncal_{(\Mcal',d')}(r',\Xcal) \ln 2 + 2 \ln \frac{1}{\delta}}{n}+2 \varsigma_{x}(w)r',
\end{align}
 where $\hat C_J\ge  \tC_{J}(w)$ for all $w\in \Wcal$ and $(\Mcal',d')$ is a metric space such that $\Xcal \subseteq \Mcal'$.
See Appendix \ref{app:1}. for more details on the algorithmic robustness bounds.

Thus, we can see that the dominant term $\Ncal_{(\Mcal',d')}(r',\Hcal)$ in Theorem \ref{thm:2} is comparable to the dominant term $\Ncal_{(\Mcal',d')}(r',\Xcal)$ in the previous study. Whereas the previous bound measures the robustness in the input space $\Xcal$, the bound in Theorem \ref{thm:2} measures   the robustness in the bottleneck layer  space $\Hcal$. When compared to  the input space $\Xcal$, if the bottleneck layer space  $\Hcal$ is smaller or has  more structures,  then we can have $\Ncal_{(\Mcal',d')}(r',\Hcal)<\Ncal_{(\Mcal',d')}(r',\Xcal)$ and  Theorem \ref{thm:2} can be advantageous over the previous bound.  However,  Theorem \ref{thm:2} is not our main result as we have much tighter bounds for the discretization process in Theorem \ref{thm:1} as well as  Theorem \ref{prop:1}. 

\subsection{On  algorithmic robustness} \label{app:1}

In the previous paper, algorithmic robustness is defined to be the measure of how much the loss value can vary with respect to the perturbations of values data points $(x,y)\in \Xcal \times \Ycal$. More precisely, an algorithm $\mathcal A$ is said to be $(|\Omega|,\varrho(\cdot))$-robust if $\mathcal{X} \times  \mathcal{Y}$ can be partitioned into $|\Omega|$ disjoint sets $\Omega_1,\dots,\Omega_{|\Omega|}$ such that for any dataset  $S \in (\mathcal{X} \times  \mathcal{Y})^m$, all $(x,y)\in S$, all $(x',y') \in\mathcal{X} \times  \mathcal{Y}$, and all $i \in \{1,\dots,|\Omega|\}$, if $(x,y),(x',y') \in \Omega_i$, then
$$
|\ell(\tf(x),y)-\ell(\tf(x'),y')|\le \varrho(S).
$$ 
If algorithm $\mathcal A$ is $(\Omega,\varrho(\cdot))$-robust and the codomain of $\ell$ is upper-bounded by $M$, then given a dataset $S$, we have \citep{xu2012robustness} that  for any $\delta>0$, with probability at least $1-\delta$, 
$$ 
\left| \EE_{x,y}[\ell(\tf(x),y)] - \frac{1}{n}\sum_{i=1}^n[\ell(\tf(x_{i}),y_{i})] \right| \le M \sqrt \frac{2| \Omega| \ln 2 + 2 \ln \frac{1}{\delta}}{n}+\varrho(S) .
$$
The previous paper \citep{xu2012robustness} further shows concrete examples of this bound for a case where the function  $(x,y)\mapsto \ell(\tf(x),y)$ is Lipschitz with Lipschitz constant $\varsigma_{x,y}(w)$, 
$$ 
\left| \EE_{x,y}[\ell(\tf(x),y)] - \frac{1}{n}\sum_{i=1}^n[\ell(\tf(x_{i}),y_{i})] \right| \le  M \sqrt \frac{2\Ncal_{(\Mcal',d')}(r',\Xcal\times \Ycal) \ln 2 + 2 \ln \frac{1}{\delta}}{n}+2 \varsigma_{x,y}(w)r',
$$
where $(\Mcal',d')$ is a metric space such that $\Xcal\times \Ycal \subseteq \Mcal'$. Note that the Lipschitz assumption on t he function $(x,y)\mapsto \ell(\tf(x),y)$  does not typically hold for the 0-1 loss on classification. For classification, we can assume that the function   $x\mapsto \ell(\tf(x),y)$ is Lipschitz instead, yielding equation \eqref{eq:15}.

\section{Proofs} \label{app:proof}
We use the notation of $q_e(h):=q(h,L,G)$ in the proofs.

\subsection{Proof of Theorem \ref{prop:1}}
\begin{proof}[Proof of Theorem \ref{prop:1}]

Let $\Ical_{k}=\{i\in[n]: q_e (h_{i})= Q_{k}\}.$ 
By using the following equality,
$$
\EE_{h}[\phi_{k}^{S}(q_e (h_{}))]=\EE_{h}[\phi_{k}^{S}(q_e (h_{}))|q_e (h)= Q_{k}]\Pr(q_e (h)= Q_{k})=\phi(Q_{k})\Pr(q_e (h)= Q_{k}),
$$
we first decompose the difference into two terms as
\begin{align} \label{eq:11}
&\EE_{h}[\phi_{k}^S(q_e(h))] - \frac{1}{n}\sum_{i=1}^n\phi_{k}^S(q_e(h_{i})) 
\\ \nonumber & =\phi(Q_{k})\left(\Pr(q_e (h)= Q_{k}) - \frac{|\Ical_{k}|}{n}\right)+\left(\phi(Q_{k})\frac{|\Ical_{_{k}}|}{n}- \frac{1}{n}\sum_{i=1}^n\phi_{k}^S(q_e(h_{i}))\right). 
\end{align}
The second term in the right-hand side  of \eqref{eq:11} is further simplified by using 
$$
\frac{1}{n}\sum_{i=1}^n\phi_{k}^S(q_e(h_{i}))=\frac{1}{n}  \sum_{i \in \Ical_{k}}\phi(q_e(h_{i})), 
$$
and
$$
\phi(Q_{k})\frac{|\Ical_{_{k}}|}{n}= \frac{1}{n}\sum_{i \in \Ical_{k}}\phi(q_e(h_{i})),
$$
as 
\begin{align*}
& \phi(Q_{k})\frac{|\Ical_{_{k}}|}{n}- \frac{1}{n}\sum_{i=1}^n\phi_{k}^S(q_e(h_{i}))=0.
\end{align*}
Substituting these into equation \eqref{eq:11} yields
\begin{align} \label{eq:14} 
\nonumber \left|\EE_{h}[\phi_{k}^S(q_e(h))] - \frac{1}{n}\sum_{i=1}^n\phi_{k}^S(q_e(h_{i}))\right| &=\left|\phi(Q_{k})\left(\Pr(q_e (h)= Q_{k}) - \frac{|\Ical_{k}|}{n}\right)\right|
\\ &\le |\phi(Q_{k})|\left|\Pr(q_e (h)= Q_{k}) - \frac{|\Ical_{k}|}{n}\right|.
\end{align}
Let $p_k=\Pr(q_e (h)= Q_{k}) $ and $\hat p= \frac{|\Ical_{k}|}{n}$. Consider the random variable $X_i=\one\{q_e (h_{i})= Q_{k}\}$ with the pushforward measure of the random variable $h_i$ under the  map $h_{i}\mapsto\one\{q_e (h_{i})= Q_{k}\} $.
Here, we have that  $X_i \in\{0,1\} \subset [0,1]$. Since  $e$ is fixed and $h_1,\dots,h_n$ are assumed to be iid,  the Hoeffding's inequality
implies the following:  for each fixed $k \in [L^G]$,   
$$
\Pr(\left|p_k -\hat p_k\right| \ge t) \le 2\exp\left(- 2nt^2 \right).
$$ 
By solving $\delta'= 2\exp\left(- 2nt^2 \right)$, this implies that for each fixed $k \in [L^G]$, for any $\delta'>0$, with probability at least $1-\delta'$,
$$
\left|p_k -\hat p_k\right| \le \sqrt{\frac{\ln(2/\delta')}{2n}}.
$$
By taking union bounds over $k \in [L^{G}]$ with $\delta'=\frac{\delta}{L^G}$, we have that  for any $\delta>0$, with probability at least $1-\delta$,
the following holds  for all $k \in [L^{G}]$:
$$
\left|p_k -\hat p_k\right| \le \sqrt{\frac{\ln(2L^{G}/\delta)}{2n}}.
$$
Substituting this into equation \eqref{eq:14} yields that 
for any $\delta>0$, with probability at least $1-\delta$,
the following holds  for all $k \in [L^{G}]$:
\begin{align*}
\left| \EE_{h}[\phi_{k}^S(q_e(h))] - \frac{1}{n}\sum_{i=1}^n\phi_{k}^S(q_e(h_{i})) \right| \le |\phi(Q_{k})| \sqrt{\frac{\ln(2L^{G}/\delta)}{2n}}= |\phi(Q_{k})| \sqrt{\frac{G\ln(L)+\ln(2/\delta)}{2n}}. 
\end{align*}
\end{proof}

\subsection{Proof of Theorem \ref{prop:2}}

\begin{proof}[Proof of Theorem \ref{prop:2}]
Let $(\Mcal',d')$ be a matric space such that $\Hcal \subseteq \Mcal'$. Fix $r'> 0$ and  $\bCcal\in \argmin_{\Ccal}\{|\Ccal|:\Ccal \subseteq\Mcal',\Hcal \subseteq \cup_{c \in\Ccal} \Bcal_{(\Mcal', d')}[c,r ']\}$ such that $|\bCcal|<\infty$. Fix an arbitrary ordering and define $c_k \in \bCcal_{r',d'}$ to be the $k$-the element in the ordered version of    $ \bCcal$ in that fixed ordering (i.e., $\cup_{k} \{c_k\}=  \bCcal_{r',d'}$). Let $ \Bcal[c]=\Bcal_{(\Mcal', d')}[c,r ']$ and $S=\{\Bcal[c_{1}], \Bcal[c_{2}],\dots, \Bcal_{}[c_{|\bCcal_{} |}]\}$. Suppose that  $\left|\phi_{k}^{S}(h{})-\phi_{k}^S(h')\right| \le\xi_{k}(r',\Mcal',d)$  for any $h,h'  \in \Bcal_{}[c_{k}]$ and $k \in [|\bCcal|]$, which is shown to be satisfied later in this proof.  Let $\Ical_{k}=\{i\in[n]:  h_{i}\in \Bcal_{}[c_{k}]\}$ for all $k \in [|\bCcal_{}|]$. 
By using the following equality,
$$
\EE_{h}[\phi_{k}^{S}(h{})]=\EE_{h}[\phi_{k}^{S}(h{})|h\in \Bcal_{}[c_{k}]]\Pr(h\in \Bcal_{}[c_{k}]),
$$
we first decompose the difference into two terms as
\begin{align} \label{eq:12}
&\left|\EE_{h}[\phi_{k}^S(h)] - \frac{1}{n}\sum_{i=1}^n\phi_{k}^S(h_{i}) \right| 
\\ \nonumber & \le \left|\EE_{h}[\phi_{k}^{S}(h{})|h\in \Bcal_{}[c_{k}]]\left(\Pr(h\in \Bcal_{}[c_{k}])- \frac{|\Ical_{k}|}{n}\right) \right|
 + \left|\EE_{h}[\phi_{k}^{S}(h{})|h\in \Bcal_{}[c_{k}]]\frac{|\Ical_{k}|}{n}- \frac{1}{n}\sum_{i=1}^n\phi_{k}^S(h_{i}) \right|
\end{align}
The second term in the right-hand side  of \eqref{eq:12} is further simplified by using 
$$
 \frac{1}{n}\sum_{i=1}^n\phi_{k}^S(h_{i}) =\frac{1}{n} \sum_{i \in \Ical_{k}}\phi_{k}^S(h_{i}), 
$$
and
$$
\EE_{h}[\phi_{k}^{S}(h{})|h\in \Bcal_{}[c_{k}]]\frac{|\Ical_{k}|}{n}= \frac{1}{n}\sum_{i \in \Ical_{k}}\EE_{h}[\phi_{k}^{S}(h{})|h\in \Bcal_{}[c_{k}]],
$$
as 
\begin{align*}
& \left|\EE_{h}[\phi_{k}^{S}(h{})|h\in \Bcal_{}[c_{k}]]\frac{|\Ical_{k}|}{n}- \frac{1}{n}\sum_{i=1}^n\phi_{k}^S(h_{i}) \right|
\\ & =\left|\frac{1}{n} \sum_{i \in \Ical_{k}}\left(\EE_{h}[\phi_{k}^{S}(h{})|h\in \Bcal_{}[c_{k}]]-\phi_{k}^S(h_{i})\right) \right|
\\ & \le \frac{1}{n} \sum_{i \in \Ical_{k}}\sup_{  h\in \Bcal_{}[c_{k}]}\left|\phi_{k}^{S}(h{})-\phi_{k}^S(h_{i})\right| \le\frac{|\Ical_k|}{n} \xi_{k}(r',\Mcal',d). 
\end{align*}
Substituting these into equation \eqref{eq:12} yields
\begin{align} \label{eq:13}
\nonumber &\left|\EE_{h}[\phi_{k}^S(h)] - \frac{1}{n}\sum_{i=1}^n\phi_{k}^S(h_{i})\right| 
\\ \nonumber & \le\left|\EE_{h}[\phi_{k}^{S}(h{})|h\in \Bcal_{}[c_{k}]]\left(\Pr(h\in \Bcal_{}[c_{k}])- \frac{|\Ical_{k}|}{n}\right) \right|
 +\frac{|\Ical_k|}{n} \xi_{k}(r',\Mcal',d)
\\  & \le \left|\EE_{h}[\phi(h{})|h\in \Bcal_{}[c_{k}]]\right|\left|\left(\Pr(h\in \Bcal_{}[c_{k}])- \frac{|\Ical_{k}|}{n}\right)\right|+\frac{|\Ical_k|}{n} \xi_{k}(r',\Mcal',d), 
\end{align}
Let $p_k=\Pr(h\in \Bcal_{}[c_{k}]) $ and $\hat p= \frac{|\Ical_{k}|}{n}$. Consider the random variable $X_i=\one\{h\in \Bcal_{}[c_{k}]\}$ with the pushforward measure of the random variable $h_i$ under the  map $h_{i}\mapsto\one\{h\in \Bcal_{}[c_{k}]\} $.
Here, we have that  $X_i \in\{0,1\} \subset [0,1]$. Since $\Bcal[c_{k}]$ is fixed and  $h_1,\dots,h_n$ are assumed to be iid,  the Hoeffding's inequality
implies the following:  for each fixed $k \in [|\bCcal_{} |]$,   
$$
\Pr(\left|p_k -\hat p_k\right| \ge t) \le 2\exp\left(- 2nt^2 \right).
$$ 
By solving $\delta'= 2\exp\left(- 2nt^2 \right)$, this implies that for each fixed $k \in [|\bCcal_{} |]$, for any $\delta'>0$, with probability at least $1-\delta'$,
$$
\left|p_k -\hat p_k\right| \le \sqrt{\frac{\ln(2/\delta')}{2n}}.
$$
By taking union bounds over $k \in [|\bCcal_{} |]$ with $\delta'=\frac{\delta}{|\bCcal_{} |}$, we have that  for any $\delta>0$, with probability at least $1-\delta$,
the following holds  for all $k \in [|\bCcal_{} |]$:
$$
\left|p_k -\hat p_k\right| \le \sqrt{\frac{\ln(2|\bCcal_{} |/\delta)}{2n}}= \sqrt{\frac{\ln(|\bCcal_{} |)+\ln(2/\delta)}{2n}}.
$$
Substituting this into equation \eqref{eq:13} yields that 
for any $\delta>0$, with probability at least $1-\delta$,
the following holds  for all $k \in [|\bCcal_{} |]$:
\begin{align*}
&\left| \EE_{h}[\phi_{k}^S(h)] - \frac{1}{n}\sum_{i=1}^n\phi_{k}^S(h_{i}) \right| 
\\ & \le \left|\EE_{h}\left[\phi(h{})|h\in \Bcal_{}[c_{k}] \right]\right| \sqrt{\frac{\ln(\Ncal_{(\Mcal',d')}(r',\Hcal))+\ln(2/\delta)}{2n}}
\\ & \hspace{10pt}+\xi_{k}(r',\Mcal',d) \left(\frac{1}{n} \sum_{i=1}^n\one\{h_{i}\in \Bcal_{}[c_{k}]\}\right), 
\end{align*}
where we used $|\Ical_k|=\sum_{i=1}^n\one\{h_{i}\in \Bcal_{}[c_{k}]\}$.
Let us now choose the metric space $(\Mcal',d')$ to be the Euclidean space $\RR^{m}$ with the Euclidean metric and $\Hcal \subset \RR^{m}$  such that $\|v\|_2 \le R_{\Hcal}$ for all $v \in \Hcal$. Then, we have  $ \Ncal_{(\Mcal',d')}(r',\Hcal)\le (2R_{\Hcal} \sqrt{m}/r')^m$ and  we can set $\xi(w,r',\Mcal',d)=2\varsigma_k r'$. This is because that  the function $h \mapsto \phi_{k}^{S}(h{})$ is Lipschitz continuous as  
$|\phi_{k}^{S}(h{})-\phi_{k}^{S}(h{}')|\le \varsigma_{k}d(h,h')$, and because $d(h,h')\le2 r'$ for any $h,h'  \in \Bcal_{(\Mcal', d')}[c_{k},r ']$. Thus, by setting $r'=R_{\Hcal}/(2\sqrt{n})$, we can replace  $ \Ncal_{(\Mcal',d')}(r',\Hcal)$  by $(4 \sqrt{nm})^m$ and  set $\xi(w,r',\Mcal',d)=\varsigma_{k}R_{\Hcal}/\sqrt{n}$.

This yields
\begin{align*}
&\left| \EE_{h}[\phi_{k}^S(h)] - \frac{1}{n}\sum_{i=1}^n\phi_{k}^S(h_{i}) \right| 
\\ & \le \left|\EE_{h}\left[\phi(h{})|h\in \Bcal_{}[c_{k}] \right]\right| \sqrt{\frac{m\ln( 4\sqrt{nm})+\ln(2/\delta)}{2n}}+\frac{\varsigma_{k}R_{\Hcal}}{\sqrt n} \left(\frac{1}{n} \sum_{i=1}^n\one\{h_{i}\in \Bcal_{}[c_{k}]\}\right). 
\end{align*}

\end{proof}

\subsection{Proof of Theorem \ref{thm:1}}
In the proof of Theorem \ref{thm:1}, we write 
$
f(x):=f(x,w,e,\theta)
$   
when
the dependency on $(w,e,\theta)$ is clear from the context.
\begin{lemma} \label{lemma:1}
Fix $\theta \in \Theta$ and $e \in E$. Then for any $\delta>0$,  with  probability at least $1-\delta$ over an
iid draw of $n$ examples $((x_i
, y_i))_{i=1}^n$, the following holds for any $w \in \Wcal$:  
\begin{align*}
\left|\EE_{x,y}[J(f(x,w,e,\theta) ,y)] - \frac{1}{n}\sum_{i=1}^nJ(f(x,w,e,\theta) ,y_{i}) \right| \le C_{J}(w) \sqrt{\frac{4L^{G} \ln2 + 2 \ln(1/\delta)}{n}}. 
\end{align*}
\end{lemma}
\begin{proof}[Proof of Lemma \ref{lemma:1}]
Let $\Ical_{k,y}=\{i\in[n]: (q_e \circ h_\theta )(x_{i})= Q_{k},y_{i}=y\}.$ 
Using $\EE_{x,y}[J(f(x) ,y)]=\sum_{k=1}^{L^G} \sum_{y'\in \Ycal}\EE_{x,y}[J(f(x) ,y)| (q_e \circ h_\theta )(x)= Q_{k},y=y']\Pr( (q_e \circ h_\theta )(x)= Q_{k}\wedge y=y')$,
we first decompose the difference into two terms as
\begin{align} \label{eq:1}
&\EE_{x,y}[J(f(x) ,y)] - \frac{1}{n}\sum_{i=1}^nJ(f(x_{i}) ,y_{i}) 
\\ \nonumber & =\scalebox{0.95}{$\displaystyle \sum_{k=1}^{L^G} \sum_{y'\in \Ycal}\EE_{x,y}[J(f(x) ,y)| (q_e \circ h_\theta )(x)= Q_{k},y=y']\left(\Pr( (q_e \circ h_\theta )(x)= Q_{k}\wedge y=y') - \frac{|\Ical_{k,y'}|}{n}\right)$}
\\ \nonumber  & \hspace{10pt} +\left(\sum_{k=1}^{L^G} \sum_{y'\in \Ycal}\EE_{x,y}[J(f(x) ,y)| (q_e \circ h_\theta )(x)= Q_{k},y=y']\frac{|\Ical_{k,y'}|}{n}- \frac{1}{n}\sum_{i=1}^nJ(f(x_{i}) ,y_{i})\right). 
\end{align}
The second term in the right-hand side  of \eqref{eq:1} is further simplified by using 
$$
\frac{1}{n}\sum_{i=1}^nJ(f(x) ,y) =\frac{1}{n}\sum_{k=1}^{L^G} \sum_{y'\in \Ycal} \sum_{i \in \Ical_{k,y'}}J(f(x_{i}) ,y_{i}), 
$$
and
\begin{align*}
&\EE_{x,y}[J(f(x) ,y)| (q_e \circ h_\theta )(x)
\\ &= Q_{k},y=y']\frac{|\Ical_{k,y'}|}{n}= \frac{1}{n}\sum_{i \in \Ical_{k,y'}}\EE_{x,y}[J(f(x) ,y)| (q_e \circ h_\theta )(x)= Q_{k},y=y'],
\end{align*}
as 
\begin{align*}
& \sum_{k=1}^{L^G} \sum_{y'\in \Ycal}\EE_{x,y}[J(f(x) ,y)| (q_e \circ h_\theta )(x)= Q_{k},y=y']\frac{|\Ical_{k,y'}|}{n}- \frac{1}{n}\sum_{i=1}^nJ(f(x_{i}) ,y_{i})
\\ & =\frac{1}{n}\sum_{k=1}^{L^G} \sum_{y'\in \Ycal} \sum_{i \in \Ical_{k,y'}}\left(\EE_{x,y}[J(f(x) ,y)| (q_e \circ h_\theta )(x)= Q_{k},y=y']-J(f(x_{i}) ,y_{i})\right)
\\ & =\frac{1}{n}\sum_{k=1}^{L^G} \sum_{y'\in \Ycal} \sum_{i \in \Ical_{k,y'}}\left(J(\varphi_w( Q_{k}) ,y')-(J(\varphi_w( Q_{k}) ,y')\right)=0
\end{align*}
Substituting these into equation \eqref{eq:1} yields
\begin{align*} 
&\EE_{x,y}[J(f(x) ,y)] - \frac{1}{n}\sum_{i=1}^nJ(f(x_{i}) ,y_{i}) 
\\ \nonumber & =\sum_{k=1}^{L^G} \sum_{y'\in \Ycal}J(\varphi_w( Q_{k}) ,y')\left(\Pr( (q_e \circ h_\theta )(x)= Q_{k}\wedge y=y') - \frac{|\Ical_{k,y'}|}{n}\right)
\\ \nonumber & =\sum_{k=1}^{2{L^G}} J(v_{k})\left(\Pr( ((q_e \circ h_\theta )(x),y)= v_{k}) - \frac{|\Ical_{k}|}{n}\right), \end{align*}
where the last line uses the fact that $\Ycal=\{y^{(1)},y^{(2)}\}$ for some $(y^{(1)},y^{(2)})$, along with the additional notation  $\Ical_{k}=\{i\in[n]: ((q_e \circ h_\theta )(x_{i}),y_{i})= v_{k}\}$. Here, $v_{k}$ is defined as   $v_{k}=(\varphi_w( Q_{k}) ,y^{(1)})$ for all $k\in[{L^G}]$ and $v_{k}=(\varphi_w( e_{k-{L^G}}) ,y^{(2)})$ for all $k \in \{{L^G}+1,\dots,2{L^G}\}$. 

By using the bound of  $|J(\varphi_w(\eta) ,y)|\le C_{J}(w)$,
\begin{align*} 
&\left|\EE_{x,y}[J(f(x) ,y)] - \frac{1}{n}\sum_{i=1}^nJ(f(x_{i}) ,y_{i})\right|
\\ & = \left|\sum_{k=1}^{2{L^G}} J(v_{k})\left(\Pr( ((q_e \circ h_\theta )(x),y)= v_{k}) - \frac{|\Ical_{k}|}{n}\right)\right|
\\ \nonumber & \le C_{J}(w) \sum_{k=1}^{2{L^G}} \left|\Pr( ((q_e \circ h_\theta )(x),y)= v_{k}) - \frac{|\Ical_{k}|}{n}\right|.
\end{align*}
Since $|\Ical_{k}|=\sum_{i=1}^n \one\{((q_e \circ h_\theta )(x_{i}),y_{i})= v_{k}\}$ and $(\theta,e)$ is fixed, the vector $(|\Ical_{1}|,\dots,|\Ical_{2{L^G}}|)$ follows  a multinomial distribution with parameters $n$ and $p=(p_1, ..., p_{2{L^G}})$, where $p_k=\Pr( ((q_e \circ h_\theta )(x),y)= v_{k})$ for $k=1,\dots,2{L^G}$. Thus, by using the Bretagnolle-Huber-Carol inequality \citep[A6.6 Proposition]{van1996}, we have that with probability at least $1-\delta$,
\begin{align*} 
\left|\EE_{x,y}[J(f(x) ,y)] - \frac{1}{n}\sum_{i=1}^nJ(f(x_{i}) ,y_{i})
  \right| \le C_{J}(w) \sqrt{\frac{4{L^G} \ln2 + 2 \ln(1/\delta)}{n}}.
\end{align*}  

\end{proof}

\begin{proof}[Proof of Theorem \ref{thm:1}] Let $\hat \Ccal_{r,d} \in \argmin_{\Ccal}\left\{|\Ccal|: \Ccal \subseteq \Mcal, E\times \Theta\subseteq \cup_{c \in\Ccal} \Bcal_{(\Mcal,d)}[c,r]\right.\}$. Note that if $\Ncal_{(\Mcal,d)}(r, E\times \Theta)=\infty$, the bound in the statement of the theorem vacuously holds.  Thus, we focus on the case of $\Ncal_{(\Mcal,d)}(r, E\times \Theta)=|\hat \Ccal_{r,d}| <\infty$.  For any $(w,e,\theta) \in \Wcal \times E \times \Theta$, the following holds: for any $(\hat e, \hat \theta) \in \hat \Ccal_{r,d}$, 
\begin{align} \label{eq:2}
\nonumber \left|\psi_{w}(e,\theta) \right| 
& =\left|\psi_{w}(\hat e, \hat \theta)+\psi_{w}(e,\theta)-\psi_{w}(\hat e, \hat \theta)  \right| 
\\ & \le\left|\psi_{w}(\hat e, \hat \theta)\right|+\left|\psi_{w}(e,\theta)-\psi_{w}(\hat e, \hat \theta)    \right|.    
\end{align}
For the first term in the right-hand side of \eqref{eq:2}, by using Lemma \ref{lemma:1} with $\delta=\delta'/\Ncal_{(\Mcal,d)}(r,E\times \Theta)$ and taking union bounds, we have that for any $\delta'>0$, with probability at least  $1-\delta'$,
the following holds for all $(\hat e, \hat \theta) \in \hat \Ccal_{r,d}$, \begin{align} \label{eq:3}
\left|\psi_{w}(\hat e, \hat \theta)\right|\le C_{J}(w) \sqrt{\frac{4{L^G} \ln2 + 2 \ln(\Ncal_{(\Mcal,d)}(r, E\times \Theta)/\delta')}{n}}.
\end{align}
By combining equations \eqref{eq:2} and \eqref{eq:3}, we have that   for any $\delta'>0$, with probability at least  $1-\delta'$,
the following holds for any $(w,e,\theta) \in \Wcal \times E \times \Theta$ and  any $(\hat e, \hat \theta) \in \hat \Ccal_{r,d}$:
\begin{align*}
\left|\psi_{w}(e,\theta) \right| 
\le C_{J}(w) \sqrt{\frac{4{L^G} \ln2 + 2 \ln(\Ncal_{(\Mcal,d)}(r, E\times \Theta)/\delta')}{n}}+\left|\psi_{w}(e,\theta)-\psi_{w}(\hat e, \hat \theta)    \right|.    
\end{align*}
This implies that for any $\delta'>0$, with probability at least  $1-\delta'$, the following holds for  any $(w,e,\theta) \in \Wcal \times E \times \Theta$:
\begin{align} \label{eq:4}
\left|\psi_{w}(e,\theta) \right| 
\le C_{J}(w) \sqrt{\frac{4{L^G} \ln2 + 2 \ln(\Ncal_{(\Mcal,d)}(r, E\times \Theta)/\delta')}{n}}+ \min_{(\hat e, \hat \theta) \in \hat \Ccal_{r,d}}\left|\psi_{w}(e,\theta)-\psi_{w}(\hat e, \hat \theta)    \right|.    
\end{align}
For the second term in the right-hand side of \eqref{eq:4}, we have that for  any $(w,e,\theta) \in \Wcal \times E \times \Theta$,
\begin{align*} 
\min_{(\hat e, \hat \theta) \in \hat \Ccal_{r,d}}\left|\psi_{w}(e,\theta)-\psi_{w}(\hat e, \hat \theta)\right| \le \Lcal_{d} (w)\min_{(\hat e, \hat \theta) \in \hat \Ccal_{r,d}} d((e,\theta), (\hat e, \hat \theta)) \le\Lcal_{d} (w)r.  
\end{align*}
Thus, by using $r=\Lcal_{d} (w)^{1/\rho-1} \sqrt{\frac{1}{n}}$, we have that
for any $\delta'>0$, with probability at least  $1-\delta'$, the following holds for  any $(w,e,\theta) \in \Wcal \times E \times \Theta$:
\begin{align} \label{eq:5}
\left|\psi_{w}(e,\theta) \right| 
\le C_{J}(w) \sqrt{\frac{4{L^G} \ln2 + 2 \ln(\Ncal_{(\Mcal,d)}(r, E\times \Theta)/\delta')}{n}}+ \sqrt{\frac{\Lcal_{d} (w)^{2/\rho}}{n}}.    
\end{align}
Since this statement holds for any $\delta'>0$, this implies the statement of this theorem.
\end{proof}

\subsection{Proof of Theorem \ref{thm:2}}
In the proof of Theorem \ref{thm:1}, we write 
$
\tf(x):=\tf(x,w,\theta)
$   
when
the dependency on $(w,\theta)$ is clear from the context.

\begin{lemma} \label{lemma:2}
Fix $\theta \in \Theta$. Let $(\Mcal',d')$ be a matric space such that $\Hcal \subseteq \Mcal'$. Fix $r'> 0$ and  $\bCcal_{r',d'} \in \argmin_{\Ccal}\{|\Ccal|:\Ccal \subseteq\Mcal',\Hcal \subseteq \cup_{c \in\Ccal} \Bcal_{(\Mcal', d')}[c,r ']\}$. Assume that for any $c \in\bCcal_{r',d'} $, we have $|(J(\varphi_w(h) ,y)-(J(\varphi_w(h') ,y)|\le  \xi(w,r',\Mcal',d)$  for any $h,h'  \in \Bcal_{(\Mcal', d')}[c,r ']$ and $y \in \Ycal$. Then for any $\delta>0$,  with  probability at least $1-\delta$ over an
iid draw of $n$ examples $((x_i
, y_i))_{i=1}^n$, the following holds for any $w \in \Wcal$:  
\begin{align*}
&\left|\EE_{x,y}[J(\tf(x,w,\theta) ,y)] - \frac{1}{n}\sum_{i=1}^nJ(\tf(x,w,\theta) ,y_{i}) \right| 
\\ & \le \tC_{J}(w) \sqrt{\frac{4\Ncal_{(\Mcal',d')}(r',\Hcal)\ln2 + 2 \ln(1/\delta)}{n}} + \xi(w,r',\Mcal',d) . 
\end{align*}
\end{lemma}
\begin{proof} [Proof of Lemma \ref{lemma:2}]
  Note that if $\Ncal_{(\Mcal',d')}(r',\Hcal)=\infty$, the bound in the statement of the theorem vacuously holds.  Thus, we focus on the case of $\Ncal_{(\Mcal',d')}(r',\Hcal)=| \bCcal_{r',d'}|  <\infty$. Fix an arbitrary ordering and define $c_k \in \bCcal_{r',d'}$ to be the $k$-the element in the ordered version of    $ \bCcal_{r',d'}$ in that fixed ordering (i.e., $\cup_{k} \{c_k\}=  \bCcal_{r',d'}$). 

Let $\Ical_{k,y}=\{i\in[n]:  h_\theta (x_{i})\in \Bcal_{(\Mcal', d')}[c_{k},r '],y_{i}=y\}$ for all $k \times y \in [|\bCcal_{r',d'}|] \times \Ycal$. 
Using $\EE_{x,y}[J(\tf(x) ,y)]=\sum_{k=1}^{|\bCcal_{r',d'}|} \sum_{y'\in \Ycal}\EE_{x,y}[J(\tf(x) ,y)|  h_\theta (x)\in \Bcal_{(\Mcal', d')}[c_{k},r '],y=y']\Pr(  h_\theta (x)\in \Bcal_{(\Mcal', d')}[c_{k},r ']\wedge y=y')$,
we first decompose the difference into two terms as
\begin{align} \label{eq:6}
&\left|\EE_{x,y}[J(\tf(x) ,y)] - \frac{1}{n}\sum_{i=1}^nJ(\tf(x_{i}) ,y_{i}) \right| 
\\ \nonumber & \scalebox{0.84}{$\displaystyle  =\left|\sum_{k=1}^{|\bCcal_{r',d'}|} \sum_{y'\in \Ycal}\EE_{x,y}[J(\tf(x) ,y)|  h_\theta (x)\in \Bcal_{(\Mcal', d')}[c_{k},r '],y=y']\left(\Pr(  h_\theta (x)\in \Bcal_{(\Mcal', d')}[c_{k},r ']\wedge y=y') - \frac{|\Ical_{k,y'}|}{n}\right) \right| $}
\\ \nonumber  & \hspace{10pt} + \left|\sum_{k=1}^{|\bCcal_{r',d'}|} \sum_{y'\in \Ycal}\EE_{x,y}[J(\tf(x) ,y)|  h_\theta (x)\in \Bcal_{(\Mcal', d')}[c_{k},r '],y=y']\frac{|\Ical_{k,y'}|}{n}- \frac{1}{n}\sum_{i=1}^nJ(\tf(x_{i}) ,y_{i}). \right|
\end{align}
The second term in the right-hand side  of \eqref{eq:6} is further simplified by using 
$$
\frac{1}{n}\sum_{i=1}^nJ(\tf(x) ,y) =\frac{1}{n}\sum_{k=1}^{|\bCcal_{r',d'}|} \sum_{y'\in \Ycal} \sum_{i \in \Ical_{k,y'}}J(\tf(x_{i}) ,y_{i}), 
$$
and
\begin{align*}
&\EE_{x,y}[J(\tf(x) ,y)|  h_\theta (x)\in \Bcal_{(\Mcal', d')}[c_{k},r '],y=y']\frac{|\Ical_{k,y'}|}{n}
\\ & = \frac{1}{n}\sum_{i \in \Ical_{k,y'}}\EE_{x,y}[J(\tf(x) ,y)|  h_\theta (x)\in \Bcal_{(\Mcal', d')}[c_{k},r '],y=y'],
\end{align*}
as 
\begin{align*}
& \left|\sum_{k=1}^{|\bCcal_{r',d'}|} \sum_{y'\in \Ycal}\EE_{x,y}[J(\tf(x) ,y)|  h_\theta (x)\in \Bcal_{(\Mcal', d')}[c_{k},r '],y=y']\frac{|\Ical_{k,y'}|}{n}- \frac{1}{n}\sum_{i=1}^nJ(\tf(x_{i}) ,y_{i})\right|
\\ & =\left|\frac{1}{n}\sum_{k=1}^{|\bCcal_{r',d'}|} \sum_{y'\in \Ycal} \sum_{i \in \Ical_{k,y'}}\left(\EE_{x,y}[J(\tf(x) ,y)|  h_\theta (x)\in \Bcal_{(\Mcal', d')}[c_{k},r '],y=y']-J(\tf(x_{i}) ,y_{i})\right) \right|
\\ & \le \frac{1}{n}\sum_{k=1}^{|\bCcal_{r',d'}|} \sum_{y'\in \Ycal} \sum_{i \in \Ical_{k,y'}}\sup_{  h\in \Bcal_{(\Mcal', d')}[c_{k},r ']}\left|J(\varphi_w(  h ,y')-J(\varphi_w(  h_\theta (x_{i})) ,y')\right| \le \xi(w). 
\end{align*}
Substituting these into equation \eqref{eq:6} yields
\begin{align*} 
&\left|\EE_{x,y}[J(\tf(x) ,y)] - \frac{1}{n}\sum_{i=1}^nJ(\tf(x_{i}) ,y_{i}) \right| 
\\ \nonumber & \le \scalebox{0.8}{$\displaystyle   \left|\sum_{k=1}^{|\bCcal_{r',d'}|} \sum_{y'\in \Ycal}\EE_{x,y}[J(\tf(x) ,y')|  h_\theta (x)\in \Bcal_{(\Mcal', d')}[c_{k},r ']]\left(\Pr(  h_\theta (x)\in \Bcal_{(\Mcal', d')}[c_{k},r ']\wedge y=y') - \frac{|\Ical_{k,y'}|}{n}\right) \right| +  \xi(w) $}
\\ \nonumber & \le \tC_{J}(w)  \sum_{k=1}^{2|\bCcal_{r',d'}|} \left|\left(\Pr( (  h_\theta (x),y)\in v_{k}) - \frac{|\Ical_{k}|}{n}\right)\right|+  \xi(w), 
\end{align*}
where the last line uses the fact that $\Ycal=\{y^{(1)},y^{(2)}\}$ for some $(y^{(1)},y^{(2)})$, along with the additional notation  $\Ical_{k}=\{i\in[n]: (  h_\theta (x_{i}),y_{i})\in  v_{k}\}$. Here, $v_{k}$ is defined as   $v_{k}= \Bcal_{(\Mcal', d')}[c_{k},r '] \times \{y^{(1)}\}$ for all $k\in[|\bCcal_{r',d'}|]$ and $v_{k}=\Bcal_{(\Mcal', d')}[c_{k-|\bCcal_{r',d'}|},r '] \times \{y^{(2)}\}$ for all $k \in \{|\bCcal_{r',d'}|+1,\dots,2|\bCcal_{r',d'}|\}$. 

Since $|\Ical_{k}|=\sum_{i=1}^n \one\{ (  h_\theta (x),y)\in v_{k}\}$ and $\theta$ is fixed, the vector $(|\Ical_{1}|,\dots,|\Ical_{2|\bCcal_{r',d'}|}|)$ follows  a multinomial distribution with parameters $n$ and $p=(p_1, ..., p_{2|\bCcal_{r',d'}|})$, where $p_k=\Pr( (  h_\theta (x),y)\in v_{k})$ for $k=1,\dots,2|\bCcal_{r',d'}|$. Thus, by noticing $|\bCcal_{r',d'}|=\Ncal_{(\Mcal',d')}(r',\Hcal)$ and by using the Bretagnolle-Huber-Carol inequality \citep[A6.6 Proposition]{van1996}, we have that with probability at least $1-\delta$,
\begin{align*} 
&\left|\EE_{x,y}[J(\tf(x) ,y)] - \frac{1}{n}\sum_{i=1}^nJ(\tf(x_{i}) ,y_{i})
  \right| 
  \\ &\ \le\tC_{J}(w) \sqrt{\frac{4\Ncal_{(\Mcal',d')}(r',\Hcal) \ln2 + 2 \ln(1/\delta)}{n}}+  \xi(w).
\end{align*}  

\end{proof}

\begin{proof}[Proof of Theorem \ref{thm:2}]
Let $\hat \Ccal_{r,d} \in \argmin_{\Ccal}\left\{|\Ccal|: \Ccal \subseteq \Mcal, \Theta\subseteq \cup_{c \in\Ccal} \Bcal_{(\Mcal,d)}[c,r]\right.\}$. Note that if $\Ncal_{(\Mcal,d)}(r,  \Theta)=\infty$, the bound in the statement of the theorem vacuously holds.  Thus, we focus on the case of $\Ncal_{(\Mcal,d)}(r,  \Theta)=|\hat \Ccal_{r,d}| <\infty$. For any $(w,\theta) \in \Wcal \times \Theta$, the following holds: for any $\hat \theta \in \hat \Ccal_{r,d}$, 
\begin{align} \label{eq:7}
\nonumber \left|\tpsi_{w}(\theta) \right| 
& =\left|\tpsi_{w}(\hat \theta)+\tpsi_{w}(\theta)-\tpsi_{w}( \hat \theta)  \right| 
\\ & \le\left|\tpsi_{w}( \hat \theta)\right|+\left|\tpsi_{w}(\theta)-\tpsi_{w}( \hat \theta)    \right|.    
\end{align}
For the first term in the right-hand side of \eqref{eq:7}, by using Lemma \ref{lemma:2} with $\delta=\delta'/\Ncal_{(\Mcal',d')}(r', \Theta)$ and taking union bounds, we have that for any $\delta'>0$, with probability at least  $1-\delta'$,
the following holds for all $ \hat \theta \in \hat \Ccal_{r,d}$, \begin{align} \label{eq:8}
\left|\tpsi_{w}( \hat \theta)\right|\le \tC_{J}(w) \sqrt{\frac{4\Ncal_{(\Mcal',d')}(r',\Hcal) \ln2 + 2 \ln(\Ncal_{(\Mcal,d)}(r,  \Theta)/\delta')}{n}}+ \xi(w,r',\Mcal',d).
\end{align}
By combining equations \eqref{eq:7} and \eqref{eq:8}, we have that   for any $\delta'>0$, with probability at least  $1-\delta'$,
the following holds for any $(w,\theta) \in \Wcal \times  \Theta$ and  any $ \hat \theta \in \hat \Ccal_{r,d}$:
\begin{align*}
&\left|\tpsi_{w}(\theta) \right| 
\\ &\le \tC_{J}(w) \sqrt{\frac{4\Ncal_{(\Mcal',d')}(r',\Hcal) \ln2 + 2 \ln(\Ncal_{(\Mcal,d)}(r,  \Theta)/\delta')}{n}}+\left|\tpsi_{w}(\theta)-\tpsi_{w}( \hat \theta)    \right|+ \xi(w,r',\Mcal',d).    
\end{align*}
This implies that for any $\delta'>0$, with probability at least  $1-\delta'$, the following holds for  any $(w,\theta) \in \Wcal \times  \Theta$:
\begin{align} \label{eq:9}
\left|\tpsi_{w}(\theta) \right| 
&\ \le C_{J}(w) \sqrt{\frac{4\Ncal_{(\Mcal',d')}(r',\Hcal) \ln2 + 2 \ln(\Ncal_{(\Mcal,d)}(r,  \Theta)/\delta')}{n}}
 \\ \nonumber  & \hspace{10pt} + \min_{ \hat \theta \in \hat \Ccal_{r,d}}\left|\tpsi_{w}(\theta)-\tpsi_{w}( \hat \theta)    \right|+ \xi(w,r',\Mcal',d).    
\end{align}
For the second term in the right-hand side of \eqref{eq:9}, we have that for  any $(w,\theta) \in \Wcal \times\Theta$,
\begin{align*} 
\min_{ \hat \theta \in \hat \Ccal_{r,d}}\left|\tpsi_{w}(\theta)-\tpsi_{w}( \hat \theta)\right| \le \tLcal_{d} (w)\min_{ \hat \theta \in \hat \Ccal_{r,d}} d(\theta,  \hat \theta) \le\Lcal_{d} (w)r.  
\end{align*}
Thus, by using $r=\tLcal_{d} (w)^{1/\rho-1} \sqrt{\frac{1}{n}}$, we have that
for any $\delta'>0$, with probability at least  $1-\delta'$, the following holds for  any $(w,\theta) \in \Wcal \times \Theta$:
\begin{align} \label{eq:10} 
\left|\tpsi_{w}(\theta) \right| 
& \le \tC_{J}(w) \sqrt{\frac{4\Ncal_{(\Mcal',d')}(r',\Hcal) \ln2 + 2 \ln(\Ncal_{(\Mcal,d)}(r,  \Theta)/\delta')}{n}}
 \\ \nonumber & \hspace{10pt} + \sqrt{\frac{\tLcal_{d} (w)^{2/\rho}}{n}}+ \xi(w,r',\Mcal',d).    
\end{align}
Since this statement holds for any $\delta'>0$, this implies the statement of this theorem.
\end{proof}

\section{Method Details} \label{app:method}

\begin{algorithm}[H]
\caption{Discretization of inter-module communication in RIM}
\SetAlgoLined
N is sample size,T is total time step, M is number of modules in the RIM model
\leavevmode \\
 initialization\;
 \For{i in 1..M}{
    initialize $z^0_i$;
}
  Training\;
 \For{n in 1..N}{
 \For{t in 1..T
 }{
 $\textsc{InputAttention}=\textsc{SoftAttention}(z^t_1,z^t_2,...,z^t_M,x^t)$;
 
    \uIf{i in top K of \textsc{InputAttention}}{
    $\hat{z}^{t+1}_{i}=\textsc{RNN}(z^{t}_{i},x^t)$;
    }\Else{$\hat{z}^{t+1}_{i'}=z^{t}_{i'}$;}
    \For{i in 1..M}{
    Discretization;
    $h^{t+1}_i=\textsc{SoftAttention}(\hat{z}^{t+1}_{1},\hat{z}^{t+1}_{2},....\hat{z}^{t+1}_{M})$
    \leavevmode \\
    $z^{t+1}_{i}=\hat{z}^{t+1}_{i}+q(h^{t+1}_i,L,G)$;
    }
 }
 Calculate task loss, codebook loss and commitment loss according to equation \ref{equation2}
 \leavevmode \\
 Update model parameter $\Theta$ together with discrete latent vectors in codebook $e \in R^{LXD}$;
 
 }
\end{algorithm}

\subsection{Task Details}

2D shape environment is a 5X5 grid world with different objects of different shapes and colors placed at random positions.Each location can only be occupied by one object.The underlying environment dynamics of 3D shapes are the same as in the 2D dateset, and only the rendering component was changed \citep{kipf2019contrastive}. In OOD setting, the total number of objects are changed for each environment. We used number of objects of 4 (validation), 3 (OOD-1)  and 2 (OOD-2). We did not put in more than 5 objects because the environment will be too packed and the objects can hardly move.

The 3-body physics simulation environment is an interacting system that evolves according to physical laws.There are no actions applied onto any objects and movement of objects only depend on interaction among objects. This environment is adapted from \citet{kipf2019contrastive}. In the training environment, the radius of each ball is 3.In OOD settings, we changed the radius to 4 ( validation) and 2 (OOD test).

In all the 8 Atari games belong to the same collections of 2600 games from Atari Corporation.We used the games adapted to OpenAI gym environment. There are several versions of the same game available in OpenAI gym. We used version "Deterministic-v0" starting at warm start frame 50 for each game for training. Version "Frameskip-v0" starting at frame 250 as OOD validation and "Frameskip-v4" starting at frame 150 at OOD test.

In all the GNN compositional reasoning experiments. HITS at RANK K (K=1 in this study) was used as as the metrics for performance.This binary score is 1 for a particular example if the predicted state representation is in the k-nearest neighbor set around the true observation. Otherwise this score is 0.  MEAN RECIPROCAL RANK (MRR) is also used as a performance metrics, which is defined as $MRR=\frac{1}{N}\sum^{N}_{n=1}\frac{1}{Rank_{n}}$ where $rank_n$ is the rank of the n-th sample (\cite{kipf2019contrastive}).

In  adding task, gap length of 500 was used for training and gap length of 200 (OOD validation) and 1000 (OOD testing) are used for OOD settings. In sequential MNIST experiment , model was trained at 14X14 resolution and tested in different resolutions (\citet{goyal2019recurrent}). Sort-of-Clevr experiments are conducted in the same way as \citet{goyal2021coordination}

\begin{figure}
    \centering
    \centering
     \begin{subfigure}[b]{0.25\textwidth}
         \centering
         \includegraphics[width=\textwidth]{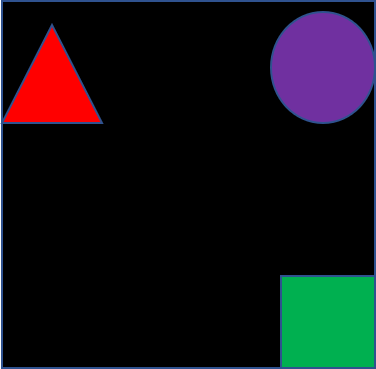}
         \caption{2D Shapes}
 \end{subfigure}
     \hfill
     \begin{subfigure}[b]{0.25\textwidth}
         \centering
         \includegraphics[width=\textwidth]{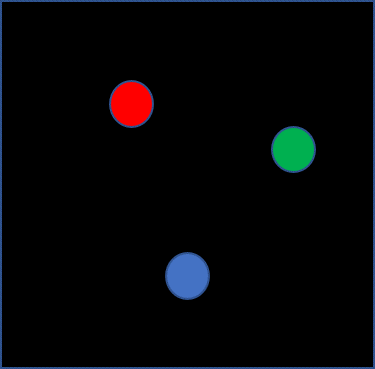}
         \caption{Three-body}
     \end{subfigure}
     \hfill
     \begin{subfigure}[b]{0.64\textwidth}
         \centering
         \includegraphics[width=\textwidth,trim={0 0 0 2cm},clip]{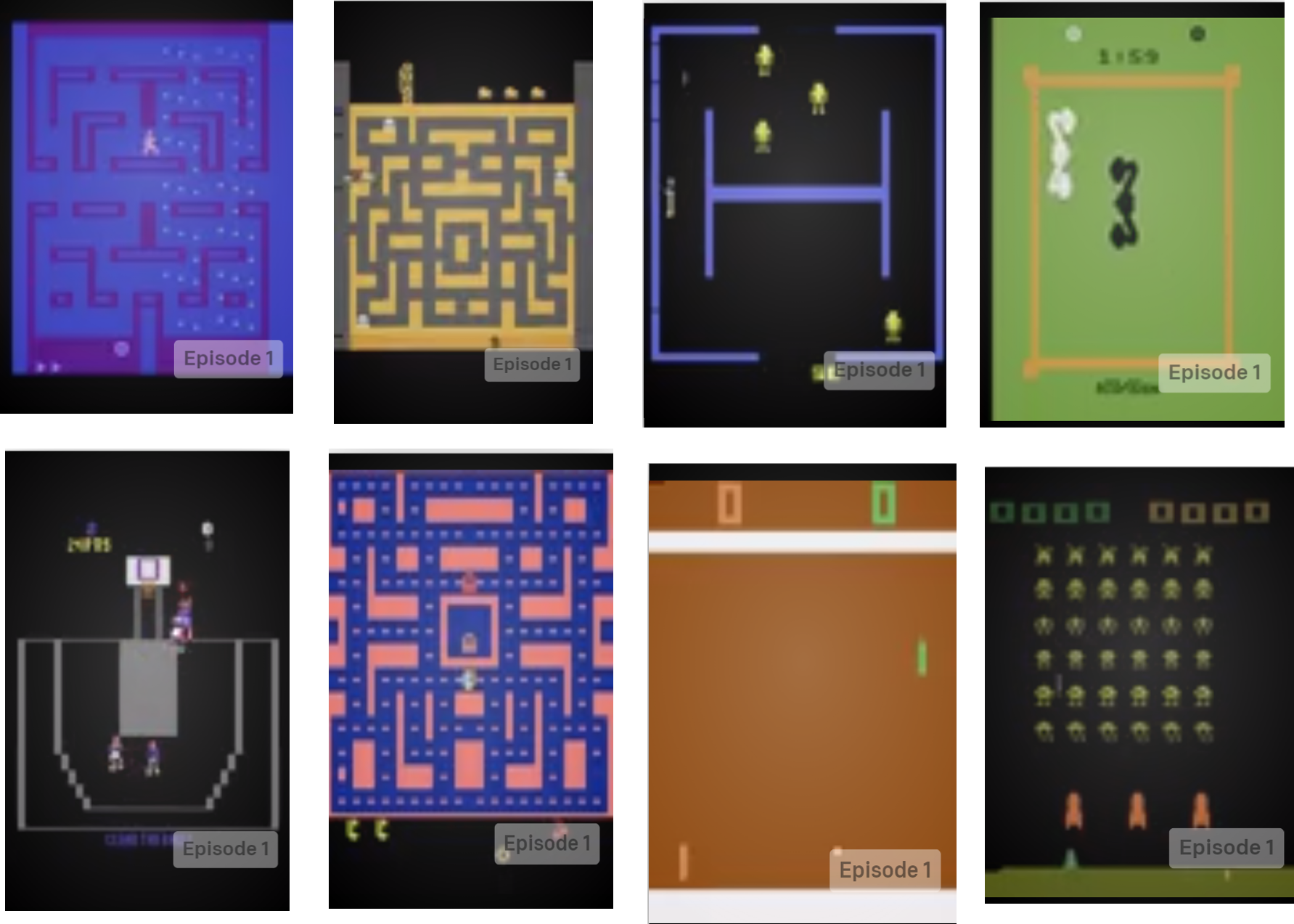}
         \caption{Atari Games}
     \end{subfigure}
     \hfill
     \begin{subfigure}[b]{0.84\textwidth}
         \centering
         \includegraphics[width=\textwidth]{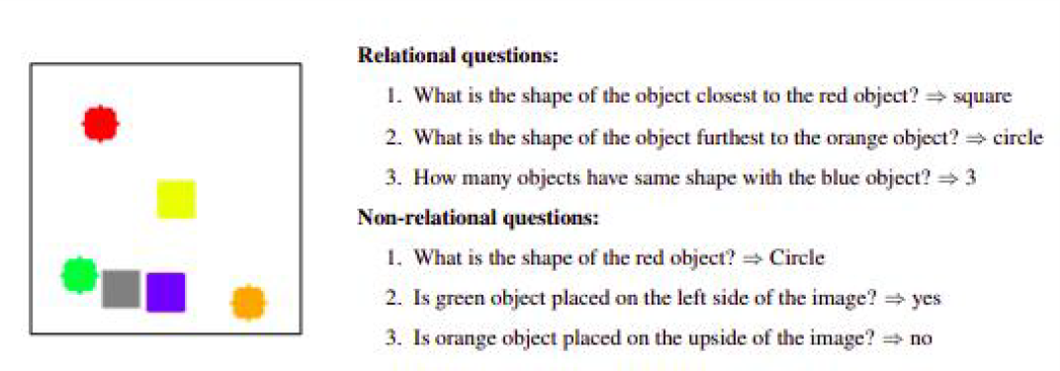}
         \caption{Sort-of-Clevr}
     \end{subfigure}
    
    \caption{ Examples of different task environments. Atari game screen shots are obtained from OpenAI gym platform. Sort-of-Clevr example was adapted from \citet{goyal2021coordination} with permission }
    \label{fig:tasks_image}
\end{figure}

\subsection{Model Architecture, Hyperparameters and Training Details }

\paragraph{DVNC implementation details} In DVNC, codebook $e \in \RR^{L \times m}$ was initialized by applying K-means clustering on training data points ($s$) where the number of clusters is $L$. The nearest $e_j$,by Euclidean distance ,is assigned to each $s_i$. The commitment loss $\beta \sum^{G}_{i}||s_{i}-\sg(e_{o_i})||^2_2$ , which encourages $s_i$ stay close to the chosen codebook vector, and the task loss are back-propagated to each of the model components that send information in the inter-component communication process.The gradient of task loss are back-propagated to each of the components that send information using straight-through gradient estimator. The codebook loss $\sum^{G}_{i}||\sg(s_{i})-e_{o_i}||^2$ that encourages  the selected codebook vector to stay close to $s_i$ is back-propagated to the selected codebook vector. Task loss is not backpropagated to codebook vectors.  Only the task loss is back-propagated to the model component that receives the information. It is worth pointing out that in this study, we train the codebook vectors directly using gradient descent instead of using exponential moving average updates as in \citet{oord2017neural}.

Model architecture, hyperparameters and training settings of GNN used in this study are same as in \citet{kipf2019contrastive}, where encoder dimension is 4 and number of object slot is 3..Model architecture, hyperparameters and training settings of RIMs used in this study are identify to \citet{goyal2019recurrent},where 6 RIM units and k=4 are used. Model architecture, hyperparameters and training settings of transformer models are the same as in \citet{goyal2021coordination}, except that we did not include shared workspace. Hyperparameters of GNN and RIM models are summarized in table \ref{table:Hyperparameters}. Hyperparameters of transformers with various settings can be found in \citet{goyal2021coordination}. In all the models mentioned above,we include discretization of communication in DVNC and keep other parts of the model unchanged.

Data are split into training set, validation set and test set, the ratio varies among different tasks depending on data availability.For in-distribution performance, validation set has the same distribution as training set. In OOD task, one of the OOD setting,eg. certain number of blocks in 2D shape experiment, is used as validation set.  The OOD setting used for validation was not included in test set.

\begin{table}[]
\caption{Hyperparameters used for GNN and RIMs}
\label{table:Hyperparameters}
\begin{tabular}{|l|l|l|l|l|}
\hline
\textbf{GNN model}       &          &  & \textbf{RIMs model}    &        \\ \hline
                         &          &  &                        &        \\ \hline
Hyperparameters          & Values   &  & Hyperparameters        & Values \\ \hline
Batch size               & 1024     &  & Batch size             & 64     \\ \hline
hidden dim               & 512      &  & hidden dim             & 300    \\ \hline
embedding-dim            & 512/G    &  & embedding-dim          & 300/G  \\ \hline
codebook\_loss\_weight   & 1        &  & codebook\_loss\_weight & 0.25   \\ \hline
Max. number of epochs    & 200      &  & Max. number of epochs  & 100    \\ \hline
Number of slots(objects) & 5        &  & learning-rate          & 0.001  \\ \hline
learning-rate            & 5.00E-04 &  & Optimizer              & Adam   \\ \hline
Optimizer                & Adam     &  & Number of Units (RIMs) & 6      \\ \hline
                         &          &  & Number of active RIMs  & 4      \\ \hline
                         &          &  & RIM unit type          & LSTM   \\ \hline
                         &          &  & dropout                & 0.5    \\ \hline
                         &          &  & gradient clipping      & 1      \\ \hline
\end{tabular}
\end{table}

\subsection{Computational resources}
GPU nodes on university cluster are used. GNN training takes ~3 hrs for each task with each hyperparameter setting on Tesla GPU. Training of RIMs and transformers take about 12 hours on the same GPU for each task. In total, the whole training progress of all models, all tasks, all hyperparameter settings takes approximately 800 hours on GPU nodes.

\end{document}